\documentclass[accepted]{article} 

\usepackage{algorithm}
\usepackage{algorithmic}
\usepackage{icml2012}
\usepackage{natbib}
\usepackage{amssymb}
\usepackage{amsmath}
\usepackage{amsthm}
\usepackage{graphicx}
\usepackage{martha}
\usepackage{xspace}
\usepackage{enumitem}
\usepackage{ifthen}
\usepackage{xspace}
\usepackage[hang,flushmargin]{footmisc} 

\newcommand{\UU}{{\mathcal{U}}}

\newcommand{\vect}[1]{\mathbf{#1}}
\newcommand{\transp}[1]{{#1}^{\mathsf{T}}}
\newcommand{\dotproduct}[2]{\transp{#1} {#2}}
\renewcommand{\exp}[1]{e^{#1}}

\newcommand{\E}[2]{\ifthenelse{\equal{#1}{}}{\mbox{E} \left[ #2\right]}{\mbox{E} \left[ #2  \middle\vert #1 \right]}}
\newcommand{\Eb}[2]{\ifthenelse{\equal{#1}{}}{\mbox{E}_{b} \left[ #2\right]}{\mbox{E}_{b} \left[ #2  \middle\vert #1 \right]}}
\newcommand{\given}{\vert}

\newcommand{\St}{\mathcal{S}}
\newcommand{\Ac}{\mathcal{A}}
\newcommand{\expR}{\mathcal{R}}
\newcommand{\obj}{J_\gamma}

\newcommand{\cp}{\vect{v}}
\newcommand{\ap}{\vect{u}}
\newcommand{\x}{\vect{x}}
\newcommand{\e}{\vect{e}}
\newcommand{\w}{\vect{w}}

\newcommand{\Ncp}{{N_{\cp}}}
\newcommand{\Nap}{{N_{\ap}}}

\newcommand{\Gret}{R^{\lambda}}
\newcommand{\deltaret}{\delta^{\lambda}}

\newcommand{\comb}{\vect{z}}
\newcommand{\vfa}{{\hat{V}}}
\newcommand{\comp}{\psi}
\newcommand{\qfa}{{\hat{Q}}}

\newcommand{\grad}[2]{\nabla_{#2} #1}
\newcommand{\gradt}[2]{{\bf g}({#2})}
\newcommand{\gradobj}{\grad{\obj(\ap)}{\ap}}
\newcommand{\gradtobj}{{\bf g}(\ap)}
\newcommand{\conv}{\mathcal{Z}}
\newcommand{\conva}{\tilde{\mathcal{Z}}}
\newcommand{\convhat}{\hat{\mathcal{Z}}}
\newcommand{\grada}{\gradt{\obj}{\ap}}

\newcommand{\seeerrata}{\footnote{See errata in section~\ref{errata}}}

\newcommand{\opac}{Off-PAC\xspace}
\title{Off-Policy Actor--Critic}

\begin{document}

\twocolumn[
\icmltitle{Off-Policy Actor-Critic}

\icmlauthor{Thomas Degris}{thomas.degris@inria.fr}
\icmladdress{Flowers Team, INRIA, Talence, ENSTA-ParisTech, Paris, France}
\icmlauthor{Martha White}{whitem@cs.ualberta.ca}
\icmlauthor{Richard S. Sutton}{sutton@cs.ualberta.ca}
\icmladdress{RLAI Laboratory, Department of Computing Science, University of Alberta, Edmonton, Canada}

\icmlkeywords{reinforcement learning}

\vskip 0.3in
]

\begin{abstract}
This paper presents the first actor-critic algorithm for off-policy reinforcement learning. Our algorithm is online and incremental, and its per-time-step complexity scales linearly with the number of learned weights. Previous work on actor-critic algorithms is limited to the on-policy setting and does not take advantage of the recent advances in off-policy gradient temporal-difference learning. Off-policy techniques, such as Greedy-GQ, enable a target policy to be learned while following and obtaining data from another (behavior) policy. For many problems, however, actor-critic methods are more practical than action value methods (like Greedy-GQ) because they explicitly represent the policy; consequently, the policy can be stochastic and utilize a large action space. In this paper, we illustrate how to practically combine the generality and learning potential of off-policy learning with the flexibility in action selection given by actor-critic methods. We derive an incremental, linear time and space complexity algorithm that includes eligibility traces, prove convergence under assumptions similar to previous off-policy algorithms\seeerrata, and empirically show better or comparable performance to existing algorithms on standard reinforcement-learning benchmark problems. 
\end{abstract}

The reinforcement learning framework is a general temporal learning formalism that has, over the last few decades, seen a marked growth in algorithms and applications. Until recently, however, practical online methods with convergence guarantees have been restricted to the \emph{on-policy} setting, in which the agent learns only about the policy it is executing. 

In an \emph{off-policy} setting, on the other hand, an agent learns about a policy or policies different from the one it is executing. Off-policy methods have a wider range of applications and learning possibilities. Unlike on-policy methods, off-policy methods are able to, for example, learn about an optimal policy while executing an exploratory policy (Sutton \& Barto, 1998), learn from demonstration (Smart \& Kaelbling, 2002), and learn multiple tasks in parallel from a single sensorimotor interaction with an environment (Sutton et al., 2011). Because of this generality, off-policy methods are of great interest in many application domains. 

The most well known off-policy method is Q-learning (Watkins \& Dayan, 1992).
However, while Q-Learning is guaranteed to converge to the optimal policy for the tabular (non-approximate) case, it may diverge when using linear function approximation (Baird, 1995). Least-squares methods such as LSTD (Bradtke \& Barto, 1996) and LSPI (Lagoudakis \& Parr, 2003) can be used off-policy and are sound with linear function approximation, but are computationally expensive; their complexity scales quadratically with the number of features and weights.
Recently, these problems have been addressed by the new family of gradient-TD (Temporal Difference) methods (e.g., Sutton et al., 2009), such as Greedy-GQ (Maei et al., 2010), which are of linear complexity and convergent under off-policy training with function approximation. 


All action-value methods, including gradient-TD methods such as Greedy-GQ, suffer from three important limitations. First, their target policies are deterministic, whereas many problems have stochastic optimal policies, such as in adversarial settings or in partially observable Markov decision processes. Second, finding the greedy action with respect to the action-value function becomes problematic for larger action spaces. Finally, a small change in the action-value function can cause large changes in the policy, which creates difficulties for convergence proofs and for some real-time applications.

The standard way of avoiding the limitations of action-value methods is to use policy-gradient algorithms (Sutton et al., 2000) such as actor-critic methods (e.g., Bhatnagar et al., 2009). For example, the natural actor-critic, an on-policy policy-gradient algorithm, has been successful for learning in continuous action spaces in several robotics applications (Peters \& Schaal, 2008).

The first and main contribution of this paper is to introduce the first actor-critic method that can be applied off-policy, which we call \opac, for Off-Policy Actor--Critic.
\opac has two learners: the actor and the critic. The actor updates the policy weights. The critic learns an off-policy estimate of the value function for the current actor policy, different from the (fixed) behavior policy. This estimate is then used by the actor to update the policy. For the critic, in this paper we consider a version of \opac that uses GTD($\lambda$) (Maei, 2011), a gradient-TD method with eligibitity traces for learning state-value functions.
We define a new objective for our policy weights and derive a valid backward-view update using eligibility traces. The time and space complexity of \opac is linear in the number of learned weights. 

The second contribution of this paper is an off-policy policy-gradient theorem and a convergence proof for \opac when $\lambda = 0$, under assumptions similar to previous off-policy gradient-TD proofs.\seeerrata 

Our third contribution is an empirical comparison of Q($\lambda$), Greedy-GQ, \opac, and a soft-max version of Greedy-GQ that we call Softmax-GQ, on three benchmark problems in an off-policy setting. To the best of our knowledge, this paper is the first to provide an empirical evaluation of gradient-TD methods for off-policy control (the closest known prior work is the work of Delp (2011)). We show that \opac outperforms other algorithms on these problems.


\section{Notation and Problem Setting}
\label{sec:notation}

In this paper, we consider Markov decision processes with a discrete state space $\St$, a discrete action space $\Ac$, a distribution $P: \St \times \St \times \Ac \ra [0,1]$, where $P(s'|s,a)$ is the probability of transitioning into state $s'$ from state $s$ after taking action $a$, and an expected reward function $\expR:\St \times \Ac \times \St \ra \Real$ that provides an expected reward for taking action $a$ in state $s$ and transitioning into $s'$. We observe a stream of data, which includes states $s_t \in \St$, actions $a_t \in \Ac$, and rewards $r_t \in \Real$ for $t=1,2,\ldots$ with actions selected from a fixed behavior policy, $b(a|s) \in (0,1]$.

Given a termination condition $\gamma:\St \rightarrow [0,1]$ (Sutton et al., 2011), we define the value function for $\pi:\St \times \Ac \rightarrow (0,1]$ to be:
\begin{align}
V^{\pi, \gamma}(s) = \E{s_t=s}{r_{t+1} + \ldots + r_{t+T}} \ \forall s \in \St\label{eq:valuefunction} 
\end{align}
where policy~$\pi$ is followed from time step $t$ and terminates at time $t+T$ according to $\gamma$. We assume termination always occurs in a finite number of steps.

The action-value function, $Q^{\pi, \gamma}(s,a)$, is defined as:
\begin{align}
Q^{\pi, \gamma}&(s,a) = \nonumber \\ 
&\sum_{s' \in \St} P(s' \given s,a) [ \expR(s,a,s') + \gamma(s') V^{\pi, \gamma}(s') ] \label{eq:actionvalue}
\end{align}
for all $\ a \in \Ac$ and for all $s \in \St$. Note that $V^{\pi,\gamma}(s)=\sum_{a \in \Ac} \pi(a|s) Q^{\pi,\gamma}(s,a)$, for all $s \in \St$.

The policy $\pi_\ap: \Ac \times \St \ra [0,1]$ is an arbitrary, differentiable function of a weight vector, $\ap \in \Real^{N_\ap}$, $N_\ap \in \Natural$, with $\pi_\ap(a|s) > 0$ for all $s \in \St$, $a \in \Ac$. Our aim is to choose $\ap$ so as to maximize the following scalar objective function:
\begin{eqnarray}
\obj(\ap) & = & \sum_{s \in \St} d^{b}(s) V^{\pi_\ap,\gamma}(s)
\label{eq:objfun}
\end{eqnarray}
where $d^{b}(s) = \lim_{t\rightarrow \infty}P(s_t=s \given s_0,b)$ is the limiting distribution of states under $b$ and $P(s_t=s \given s_0,b)$ is the probability that $s_t=s$ when starting in $s_0$ and executing~$b$. The objective function is weighted by $d^b$ because, in the off-policy setting, data is obtained according to this behavior distribution. For simplicity of notation, we will write $\pi$ and implicitly mean $\pi_\ap$. 

\section{The \opac Algorithm}
\label{sec:algorithm}

In this section, we present the \opac algorithm in three steps. First, we explain the basic theoretical ideas underlying the gradient-TD methods used in the critic. Second, we present our off-policy version of the policy-gradient theorem. Finally, we derive the forward view of the actor and convert it to a backward view to produce a complete mechanistic algorithm using eligibility traces.

\subsection{The Critic: Policy Evaluation}
\label{sec:critic}

Evaluating a policy $\pi$ consists of learning its value function, $V^{\pi, \gamma}(s)$, as defined in Equation~\ref{eq:valuefunction}. Since it is often impractical to explicitly represent every state $s$, we learn a linear approximation of $V^{\pi, \gamma}(s)$: $\vfa(s) = \dotproduct{\cp}{\x_s}$ where $\x_s \in \Real^{\Ncp}$, $\Ncp \in \Natural$, is the feature vector of the state $s$, and $\cp \in \Real^{\Ncp}$ is another weight vector. 

Gradient-TD methods (Sutton et al., 2009) incrementally learn the weights, $\cp$, in an off-policy setting, with a guarantee of stability and a linear per-time-step complexity. These methods minimize the $\lambda$-weighted mean-squared projected Bellman error:
\begin{equation*}
\mbox{MSPBE}(\cp)=||\vfa - \Pi T^{\lambda, \gamma}_{\pi} \vfa||_{{D}}^2
\end{equation*}
where $\vfa = X \cp$; ${X}$ is the matrix whose rows are all $\x_s$; $\lambda$ is the decay of the eligibility trace; $D$ is a matrix with $d^{b}(s)$ on its diagonal; $\Pi$ is a projection operator that projects a value function to the nearest representable value function given the function approximator; and $T^{\lambda, \gamma}_{\pi}$ is the $\lambda$-weighted Bellman operator for the target policy $\pi$ with termination probability $\gamma$ (e.g., see Maei \& Sutton, 2010). For a linear representation, $\Pi = {X} (\transp{{X}} {D} {X})^{-1} \transp{{X}} {D}$. 

In this paper, we consider the version of \opac that updates its critic weights by the GTD($\lambda$) algorithm introduced by Maei (2011).

\newcommand{\st}{s_t}
\newcommand{\at}{a_t}
\newcommand{\stp}{s_{t+1}}
\newcommand{\rtp}{r_{t+1}}

\subsection{Off-policy Policy-gradient Theorem}
\label{ref:policyimprovement}


Like other policy gradient algorithms, \opac updates the weights approximately in proportion to the gradient of the objective:
\begin{equation}
\ap_{t+1}  - \ap_t \approx \alpha_{u,t} \nabla_{\ap}\obj(\ap_t) \label{eqn:gradobj}
\end{equation}
where $\alpha_{u,t} \in \Real$ is a positive step-size parameter. Starting from Equation \ref{eq:objfun}, the gradient can be written:
\begin{align*}
\gradobj & = \grad{\left[ \sum_{s \in \St} d^{b}(s) \sum_{a \in \Ac} \pi(a \given s) Q^{\pi, \gamma}(s,a)\right]}{\ap} \nonumber\\
& = \sum_{s \in \St} d^{b}(s) \sum_{a \in \Ac} \left[ \grad{\pi(a \given s)}{\ap} Q^{\pi, \gamma}(s,a) \right. \nonumber \\ 
& \hspace{3cm} + \pi(a \given s) \grad{Q^{\pi, \gamma}(s,a)}{\ap} \left. \right]
\end{align*}
The final term in this equation, $\grad{Q^{\pi, \gamma}(s,a)}{\ap}$, is difficult to estimate in an incremental off-policy setting. The first approximation involved in the theory of \opac is to omit this term. That is, we work with an approximation to the gradient, which we denote $\gradtobj\in\Real^\Nap$, defined by
\begin{equation}
\label{eq:g}
\gradobj\approx\gradtobj=\sum_{s \in \St} d^{b}(s) \sum_{a \in \Ac} \grad{\pi(a|s)}{\ap} Q^{\pi,\gamma}(s,a)
\end{equation}
%
%
%
The two theorems below provide justification for this approximation\seeerrata.


\begin{theorem}[Policy Improvement]\label{thm:policyimprovement}
Given any policy parameter $\ap$, let 
\[
\ap' = \ap + \alpha \, \gradtobj
\]
Then there exists an $\epsilon>0$ such that, for all positive $\alpha<\epsilon$, 
\begin{align*}
\obj(\ap') \ge \obj(\ap)
\end{align*}
Further, if $\pi$ has a tabular representation (i.e., separate weights for each state), then
$V^{\pi_{\ap'},\gamma}(s) \ge V^{\pi_{\ap},\gamma}(s)$ for all $s \in \St$.
\end{theorem}
(Proof in Appendix$^3$).

In the conventional on-policy theory of policy-gradient methods, the policy-gradient theorem (Marbach \& Tsitsiklis, 1998; Sutton et al., 2000) establishes the relationship between the gradient of the objective function and the expected action values. In our notation, that theorem essentially says that our approximation is exact, that $\gradtobj = \gradobj$. Although, we can not show this in the off-policy case, we can establish a relationship between the solutions found using the true and approximate gradient:

\begin{theorem}[Off-Policy Policy-Gradient Theorem]\label{thm:policygradient}
Given $\UU \subset \Real^{\Nap}$ a non-empty, compact set, let 
\begin{align*}
\vspace{-0.5cm}
\conva &= \{ \ap \in \UU \ | \ \grada= 0 \}\\
\conv &= \{\ap \in \UU \ | \ \grad{\obj}{\ap}(\ap) = 0\}
\end{align*} 
where $\conv$ is the true set of local maxima and $\conva$ the set of local maxima obtained from using the approximate gradient, $\grada$. If the value function can be represented by our function class, then $ \conv \subset \conva$.
Moreover, if we use a tabular representation for $\pi$, then $ \conv = \conva$.
\end{theorem}
(Proof in Appendix$^3$).

The proof of Theorem \ref{thm:policygradient}, showing that $\conv = \conva$, requires tabular $\pi$ to avoid update overlap: updates to a single parameter influence the action probabilities for only one state. Consequently, both parts of the gradient (one part with the gradient of the policy function and the other with the gradient of the action-value function) locally greedily change the action probabilities for only that one state. Extrapolating from this result, in practice, more generally a local representation for $\pi$ will likely suffice, where parameter updates influence only a small number of states. Similarly, in the non-tabular case, the claim will likely hold if $\gamma$ is small (the return is myopic), again because changes to the policy mostly affect the action-value function locally.

Fortunately, from an optimization perspective, for all $\ap \in \conva \backslash \conv$, $\obj(\ap) < \min_{\ap' \in \conv} \obj(\ap')$, in other words, $\conv$ represents all the largest local maxima in $\conva$ with respect to the objective, $\obj$. Local optimization techniques, like random restarts, should help ensure that we converge to larger maxima and so to $\ap \in \conv$. Even with the true gradient, these approaches would be incorporated into learning because our objective, $\obj$,  is non-convex.




\subsection{The Actor: Incremental Update Algorithm with Eligibility Traces}
\label{sec:algo}
We now derive an incremental update algorithm using observations sampled from the behavior policy. First, we rewrite Equation~\ref{eq:g} as an expectation:
\begin{align*}
\gradtobj & = \E{s \sim d^b}{\sum_{a \in \Ac} \grad{\pi(a|s)}{\ap} Q^{\pi,\gamma}(s,a)}\\
& = \E{s \sim d^b}{\sum_{a \in \Ac} b(a|s) \frac{\pi(a|s)}{b(a|s)} \frac{\grad{\pi(a|s)}{\ap}}{\pi(a|s)} Q^{\pi,\gamma}(s,a)}\\
& = \E{s \sim d^b, a \sim b(\cdot|s)}{\rho(s,a) \comp(s,a) Q^{\pi,\gamma}(s,a)}\\
& = \Eb{}{\rho(s_t,a_t) \comp(s_t,a_t) Q^{\pi,\gamma}(s_t,a_t)}
\end{align*}
where $\rho(s,a) = \frac{\pi(a|s)}{b(a|s)}$, $\comp(s,a)=\frac{\grad{\pi(a|s)}{\ap}}{\pi(a|s)}$, and we introduce the new notation $\Eb{}{\cdot}$ to denote the expectation implicitly conditional on all the random variables (indexed by time step) being drawn from their limiting stationary distribution under the behavior policy. A standard result (e.g., see Sutton et al., 2000) is that an arbitrary function of state can be introduced into these equations as a baseline without changing the expected value. We use the approximate state-value function provided by the critic, $\vfa$, in this way:
\begin{align*}
&\gradtobj = \Eb{}{\rho(s_t,a_t) \comp(s_t,a_t) \left( Q^{\pi,\gamma}(s_t,a_t) - 
\vfa(s_t) \right)}
\end{align*}
The next step is to replace the action value, $Q^{\pi,\gamma}(s_t,a_t)$, by the off-policy $\lambda$-return. Because these are not exactly equal, this step introduces a further approximation:
\begin{align*}
&\gradtobj \approx \widehat\gradtobj = \Eb{}{\rho(s_t,a_t) \comp(s_t,a_t) \left( \Gret_t - 
\vfa(s_t) \right)}
\end{align*}
where the off-policy $\lambda$-return is defined by:
\begin{align*}
&\Gret_t=r_{t+1} + (1 - \lambda) \gamma(s_{t+1}) \vfa(s_{t+1}) \\
&~~~~~~~~~~~~~~+  \lambda \gamma(s_{t+1}) \rho(s_{t+1},a_{t+1}) \Gret_{t+1} 
\end{align*}
Finally, based on this equation, we can write the forward view of \opac:
\begin{align*}
& \ap_{t+1} - \ap_t = \alpha_{u,t} \rho(s_t,a_t) \comp(s_t,a_t) \left( \Gret_t - 
\vfa(s_t) \right)
\end{align*}
The forward view is useful for understanding and analyzing algorithms, but for a mechanistic implementation it must be converted to a backward view that does not involve the $\lambda$-return.
The key step, proved in the appendix, is the observation that
\begin{align}
\Eb{}{\rho(s_t,a_t) \comp(s_t,a_t) \left(\Gret_t - \vfa(s_t) \right)} = \Eb{}{\delta_t \e_{t}} \label{eq:backwardview}
\end{align}
where $\delta_t= r_{t+1} + \gamma(s_{t+1}) \vfa(s_{t+1}) - \vfa(s_t)$ is the conventional temporal difference error, and $\e_{t}\in\Real^\Nap$ is the eligibility trace of $\comp$, updated by:
\begin{equation*}
\e_{t} = \rho(s_t,a_t) \left( \comp(s_t,a_t) + \lambda \e_{t-1} \right)
\label{eq:policyeligibilityupdate}
\end{equation*}
Finally, combining the three previous equations, the backward view of the actor update can be written simply as:
\begin{align*}
& \ap_{t+1} - \ap_t = \alpha_{u,t} \delta_t \e_t
\end{align*}

\comment{
\begin{algorithm}[tb]
   \caption{The \opac algorithm}
   \label{alg:opac}
\begin{algorithmic}
   \STATE Initialize the vectors $\e_{v,0}$, $\e_{u,0}$, and $\w_0$ to zero
   \STATE Initialize the vectors $\cp_0$ and $\ap_0$ arbitrarily
   \STATE On every time step $t$, given $\x_t$, $\at$, $r_{t+1}$, $\x_{t+1}$, and the termination probabilities $\gamma_t$ and $\gamma_{t+1}$, do:
   \STATE $\delta_t = r_{t+1} + \gamma_{t+1} \dotproduct{\cp_t}{\x_{t+1}} - \dotproduct{\cp_t}{\x_t}$
   \STATE $\rho_t = \pi_t(a_t|s_t) / b(a_t|s_t)$
\\Update of the critic (GTD($\lambda$) algorithm):
   \STATE $\e_{v,t+1} = \rho_t \left[ \x_t + \gamma_t \lambda \e_{v,t} \right]$
   \STATE $\cp_{t+1} = \cp_t + \alpha_{v,t} \left[ \delta_t \e_{v,t+1} - \right.$ 
   \STATE \hspace{25mm} $\left. \gamma_{t+1} (1- \lambda)(\dotproduct{\e_{v,t+1}}{\w_t}) \x_t \right]$
   \STATE $\w_{t+1} = \w_t + \alpha_{w,t} \left[ \delta_t \e_{v,t+1} - (\dotproduct{\w_t}{\x_t}) \x_t \right]$
\\Update of the actor:
   \STATE $\e_{u,t+1} = \rho_t \left[ \frac{\grad{\pi_t(a_t|s_t)}{\ap_t}}{\pi_t(a_t|s_t)} + \gamma_t \lambda \e_{u,t} \right]$
   \STATE $\ap_{t+1} = \ap_t + \alpha_{u,t} \delta_t \e_{u,t+1}$
\end{algorithmic}
\end{algorithm}
}
\begin{algorithm}[tb]
   \caption{The \opac algorithm}
   \label{alg:opac}
\begin{algorithmic}
   \STATE Initialize the vectors $\e_{v}$, $\e_{u}$, and $\w$ to zero
   \STATE Initialize the vectors $\cp$ and $\ap$ arbitrarily
   \STATE Initialize the state $s$ 
   \STATE For each step:
   \STATE ~~~~~Choose an action, $a$, according to $b(\cdot|s)$
   \STATE ~~~~~Observe resultant reward, $r$, and next state, $s'$
   \STATE ~~~~~$\delta \leftarrow r + \gamma(s') \dotproduct{\cp}{\x_{s'}} - \dotproduct{\cp}{\x_s}$
   \STATE ~~~~~$\rho \leftarrow \pi_\ap(a|s) / b(a|s)$
   \STATE ~~~~~Update the critic (GTD($\lambda$) algorithm):
   \STATE ~~~~~~~~~~$\e_{v} \leftarrow \rho \left( \x_s + \gamma(s) \lambda \e_{v} \right)$
   \STATE ~~~~~~~~~~$\cp \leftarrow \cp + \alpha_{v} \left[ \delta \e_{v} - \gamma(s') (1- \lambda)(\dotproduct{\w}{\e_{v}}) \x_s \right]$
   \STATE ~~~~~~~~~~$\w \leftarrow \w + \alpha_{w} \left[ \delta \e_v - (\dotproduct{\w}{\x_s}) \x_s \right]$
   \STATE ~~~~~Update the actor:
   \STATE ~~~~~~~~~~$\e_{u} \leftarrow \rho \left[ \frac{\grad{\pi_\ap(a|s)}{\ap}}{\pi_\ap(a|s)} + \gamma(s) \lambda \e_{u} \right]$
   \STATE ~~~~~~~~~~$\ap \leftarrow \ap + \alpha_{u} \delta \e_{u}$
   \STATE ~~~~~$s \leftarrow s'$
\end{algorithmic}
\end{algorithm}

The complete \opac algorithm is given above as Algorithm~\ref{alg:opac}. 
Note that although the algorithm is written in terms of states $s$ and $s'$, it really only ever needs access to the corresponding feature vectors, $\x_s$ and $\x_{s'}$, and to the behavior policy probabilities, $b(\cdot|s)$, for the current state. 
All of these are typically available in large-scale applications with function approximation.
Also note that \opac is fully incremental and has per-time step computation and memory complexity that is linear in the number of weights, $N_\ap + N_\cp$.

With discrete actions, a common policy distribution is the Gibbs distribution, which uses a linear combination of features $\pi(a|s) = \frac{\exp{\dotproduct{\ap}{\phi_{s,a}}}}{\sum_b{\exp{\dotproduct{\ap}{\phi_{s,b}}}}}$ where $\phi_{s,a}$ are state-action features for state $s$, action $a$, and where $\comp(s,a) = \frac{\grad{\pi(a|s)}{\ap}}{\pi(a|s)}=\phi_{s,a} - \sum_b{\pi(b|s)\phi_{s,b}}$. The state-action features, $\phi_{s,a}$, are potentially unrelated to the feature vectors $\x_s$ used in the critic.


\setlength{\abovedisplayskip}{0.1cm}

\section{Convergence Analysis}
\label{sec:analysis}

Our algorithm has the same recursive stochastic form as the off-policy value-function algorithms
\begin{align*}
\ap_{t+1} = \ap_t + \alpha_t(h(\ap_t, \cp_t) + M_{t+1})
\end{align*}
where $h: \Real^N \ra \Real^N$ is a differentiable function
and $\{ M_t \}_{t \ge 0}$ is a noise sequence. 
Following previous off-policy gradient proofs (Maei, 2011), we study the behavior of the ordinary differential equation
\begin{align*}
\dot \ap(t) = \ap(h(\ap(t),\cp))
\end{align*}
The two updates (for the actor and for the critic) are not independent on each time step; we analyze two separate ODEs using a two timescale analysis (Borkar, 2008). The actor update is analyzed given fixed critic parameters, and vice versa, iteratively (until convergence). We make the following assumptions.

\begin{enumerate}[label=(A\arabic*)] 
\item The policy viewed as a function of $\ap$, $\pi_{(\cdot)}(a | s) : \Real^{N_\ap} \ra (0,1]$, is continuously differentiable, $\forall s \in \St, a \in \Ac$. 
\item The update on $\ap_t$ includes a projection operator, $\Gamma: \Real^{N_\ap} \ra \Real^{N_\ap}$, that projects any $\ap$ to a compact set $\UU = \{ \ap \ | \ q_i(\ap) \le 0, i = 1, \ldots, s\} \subset \Real ^{N_\ap}$, where $q_i(\cdot): \Real^{N_\ap} \ra \Real$ are continuously differentiable functions specifying the constraints of the compact region. For $\ap$ on the boundary of $\UU$, the gradients of the active $q_i$ are linearly independent. Assume the compact region is large enough to contain at least one (local) maximum of $\obj$. \label{req:projection}
\item The behavior policy has a minimum positive value $b_{\text{min}} \in (0,1]$: $b(a|s) \ge b_{\text{min}}$ $\forall s \in \St, a \in \Ac$
\item The sequence $(\x_t, \x_{t+1}, r_{t+1})_{t \ge 0}$ is i.i.d. and has uniformly bounded second moments.
\item For every $\ap \in \UU$ (the compact region to which $\ap$ is projected), $V^{\pi,\gamma}: S \ra \Real$ is bounded.\label{req:Alast}
\end{enumerate}

\textbf{Remark 1:} It is difficult to prove the boundedness of the iterates without the projection operator.
Since we have a bounded function (with range $(0,1]$), we could instead assume that the gradient goes to zero exponentially as $\ap \ra \infty$, ensuring boundedness. Previous work, however, has illustrated that the stochasticity in practice makes convergence to an unstable equilibrium unlikely (Pemantle, 1990); therefore, we avoid restrictions on the policy function and do not include the projection in our algorithm 



Finally, we have the following (standard) assumptions on features and step-sizes.
\begin{enumerate}[label=(P\arabic*)] 
\vspace{-0.3cm}
\item $||\x_t||_\infty < \infty, \ \forall t$, where $\x_t \in \Real^{\Ncp}$
\item Matrices $C = E[\x_t \transp{\x_t}]$, $A = E[\x_t(\x_t - \gamma \transp{\x_{t+1})}]$ are non-singular and uniformly bounded. $A$, $C$ and $E[r_{t+1} \x_t]$ are well-defined because the distribution of $(\x_t, \x_{t+1}, r_{t+1})$ does not depend on $t$.
\end{enumerate}

\begin{enumerate}[label=(S\arabic*)] 
\vspace{-0.5cm}
\item $\alpha_{v,t}, \alpha_{w,t}, \alpha_{u,t} > 0, \ \forall t$ are deterministic such that $\sum_t \alpha_{v,t} = \sum_t \alpha_{w,t} = \sum_t \alpha_{u,t} = \infty$ and $\sum_t \alpha_{v,t}^2 < \infty$, $\sum_t \alpha_{w,t}^2 < \infty$ and $\sum_t \alpha_{u,t}^2 < \infty$ with $\frac{\alpha_{u,t}}{\alpha_{v,t}} \ra 0$.
\item Define $ H(A) \doteq (A+ \transp{A})/2$ and let $\lambda_\text{min}(C^{-1} H(A))$ be the minimum eigenvalue of the matrix $C^{-1}H(A)$\footnote{Minimum exists as all eigenvalues real-valued (Lemma \ref{lem:gtd})}. Then $\alpha_{w,t} = \eta \alpha_{v,t}$ for some $ \eta > \max(0, -\lambda_\text{min}(C^{-1} H(A)))$.
\end{enumerate}

\textbf{Remark 2:}  The assumption $\alpha_{u,t}/\alpha_{v,t} \ra 0$ in (S1) states that the actor step-sizes go to zero at a faster rate than the value function step-sizes: the actor update moves on a slower timescale than the critic update (which changes more from its larger step sizes). This timescale is desirable because we effectively want a converged value function estimate for the current policy weights, $\ap_t$. Examples of suitable step sizes are $\alpha_{v,t} = \frac{1}{t}, \ \alpha_{u,t} = \frac{1}{1+t \log t}$ or $\alpha_{v,t} = \frac{1}{t^{2/3}}, \ \alpha_{u,t} = \frac{1}{t}$.
 (with $\alpha_{w,t} = \eta \alpha_{v,t}$ for $\eta$ satisfying (S2)).

The above assumptions are actually quite unrestrictive. Most algorithms inherently assume bounded features with bounded value functions for all policies; unbounded values trivially result in unbounded value function weights. Common policy distributions are smooth, making $\pi(a|s)$ continuously differentiable in $\ap$.  The least practical assumption is that the tuples ($\x_t, \x_{t+1}, r_{t+1})$ are i.i.d., in other words, Martingale noise instead of Markov noise. For Markov noise, our proof as well as the proofs for GTD($\lambda$) and GQ($\lambda$), require Borkar's (2008) two-timescale theory to be extended to Markov noise (which is outside the scope of this paper). Finally, the proof for Theorem \ref{thm:convergence} assumes $\lambda = 0$, but should extend to $\lambda > 0$ similarly to GTD$(\lambda)$ (see Maei, 2011, Section 7.4, for convergence remarks). 

We give a proof sketch of the following convergence theorem\seeerrata, with the full proof in the appendix. 
\begin{theorem}[Convergence of \opac]\label{thm:convergence}
Let $\lambda = 0$ and consider the \opac iterations with GTD(0)\footnote{GTD(0) is GTD($\lambda$) with $\lambda = 0$, not the different algorithm called GTD(0) by Sutton, Szepesvari \& Maei (2008)} for the critic. Assume that (A1)-\ref{req:Alast}, (P1)-(P2) and (S1)-(S2) hold. Then the policy weights, $\ap_t$, converge to $\convhat = \{ u \in \UU \ |\   \widehat\gradtobj = 0\}$ and the value function weights, $\cp_t$, converge to the corresponding TD-solution with probability one. 
\end{theorem}
\vspace{-0.4cm}
\begin{proofsketch}
We follow a similar outline to the two timescale analysis for on-policy policy gradient actor-critic (Bhatnagar et al., 2009) and for nonlinear GTD (Maei et al., 2009). We analyze the dynamics for our two weights, $\ap_t$ and $\transp{\comb_t} = (\transp{\wvec_t} \transp{\cp_t})$, based on our update rules. The proof involves satisfying seven requirements from Borkar (2008, p.~64) to ensure convergence to an asymptotically stable equilibrium. 
\vspace{-0.15cm}
\end{proofsketch}

\section{Empirical Results}
\label{sec:empiricalresults}

This section compares the performance of \opac to three other off-policy algorithms with linear memory and computational complexity: 1) Q($\lambda$) (called Q-Learning when $\lambda=0$), 2) Greedy-GQ (GQ($\lambda$) with a greedy target policy), and 3) Softmax-GQ (GQ($\lambda$) with a Softmax target policy). The policy in \opac is a Gibbs distribution as defined in section~\ref{sec:algo}. 

\newcommand{\na}{\small na}
\newcommand{\algoname}[1]{\tiny{#1}\xspace}
\newcommand{\sm}[1]{${\scriptstyle #1}$}
\newcommand{\asympt}{\tiny{final}\xspace}
\newcommand{\overa}{\tiny{overall}\xspace}

\newcommand{\cl}{@{\hspace{1mm}}r@{\hspace{1mm}}}

We used three benchmarks: mountain car, a pendulum problem and a continuous grid world. These problems all have a discrete action space and a continuous state space, for which we use function approximation. The behavior policy is a uniform distribution over all the possible actions in the problem for each time step. Note that Q($\lambda$) may not be stable in this setting (Baird, 1995), unlike all the other algorithms.

The goal of the mountain car problem (see Sutton \& Barto, 1998) is to drive an underpowered car to the top of a hill. The state of the system is composed of the current position of the car (in $[-1.2, 0.6]$) and its velocity (in $[-.07, .07]$). The car was initialized with a position of -0.5 and a velocity of 0. Actions are a throttle of $\{-1$, $0$, $1\}$. The reward at each time step is~$-1$. An episode ends when the car reaches the top of the hill on the right or after 5,000 time steps.

The second problem is a pendulum problem (Doya, 2000). The state of the system consists of the angle (in radians) and the angular velocity (in $[-78.54, 78.54]$) of the pendulum. Actions, the torque applied to the base, are $\{-2$, $0$, $2\}$. The reward is the cosine of the angle of the pendulum with respect to its fixed base. The pendulum is initialized with an angle and an angular velocity of 0 (i.e., stopped in a horizontal position). An episode ends after 5,000 time steps.

For the pendulum problem, it is unlikely that the behavior policy will explore the optimal region where the pendulum is maintained in a vertical position. Consequently, this experiment illustrates which algorithms make best use of limited behavior samples. 

The last problem is a continuous grid-world. The state is a 2-dimensional position in $[0,1]^2$. The actions are the pairs $\{(0.0, 0.0)$, $(-.05, 0.0)$, $(.05, 0.0)$, $(0.0, -.05)$, $(0.0, .05)\}$, representing moves in both dimensions. Uniform noise in $[-.025,.025]$ is added to each action component. The reward at each time step for arriving in a position $(p_x, p_y)$ is defined as: 
$
-1 + -2 ({\cal N}(p_x, .3, .1) \cdot {\cal N}(p_y, .6, .03) + {\cal N}(p_x,.4, .03) \cdot {\cal N}(p_y, .5, .1) + {\cal N}(p_x, .8, .03) \cdot {\cal N}(p_y, .9, .1))
$ where ${\cal N}(p, \mu, \sigma) = \exp{-\frac{(p - \mu)^2}{2\sigma^2}}/ \sigma \sqrt{2\pi}$. The start position is $(0.2, 0.4)$ and the goal position is $(1.0, 1.0)$. An episode ends when the goal is reached, that is when the distance from the current position to the goal is less than 0.1 (using the L1-norm), or after 5,000 time steps. Figure~\ref{fig:discretecontinusgridworldenv} shows a representation of the problem.

\newcommand{\gridworld}[1]{\frame{\includegraphics[width=.4\linewidth]{#1}}}
\begin{figure}[t]
\centering
\begin{tabular}{cc}
Behavior & Greedy-GQ \\
\gridworld{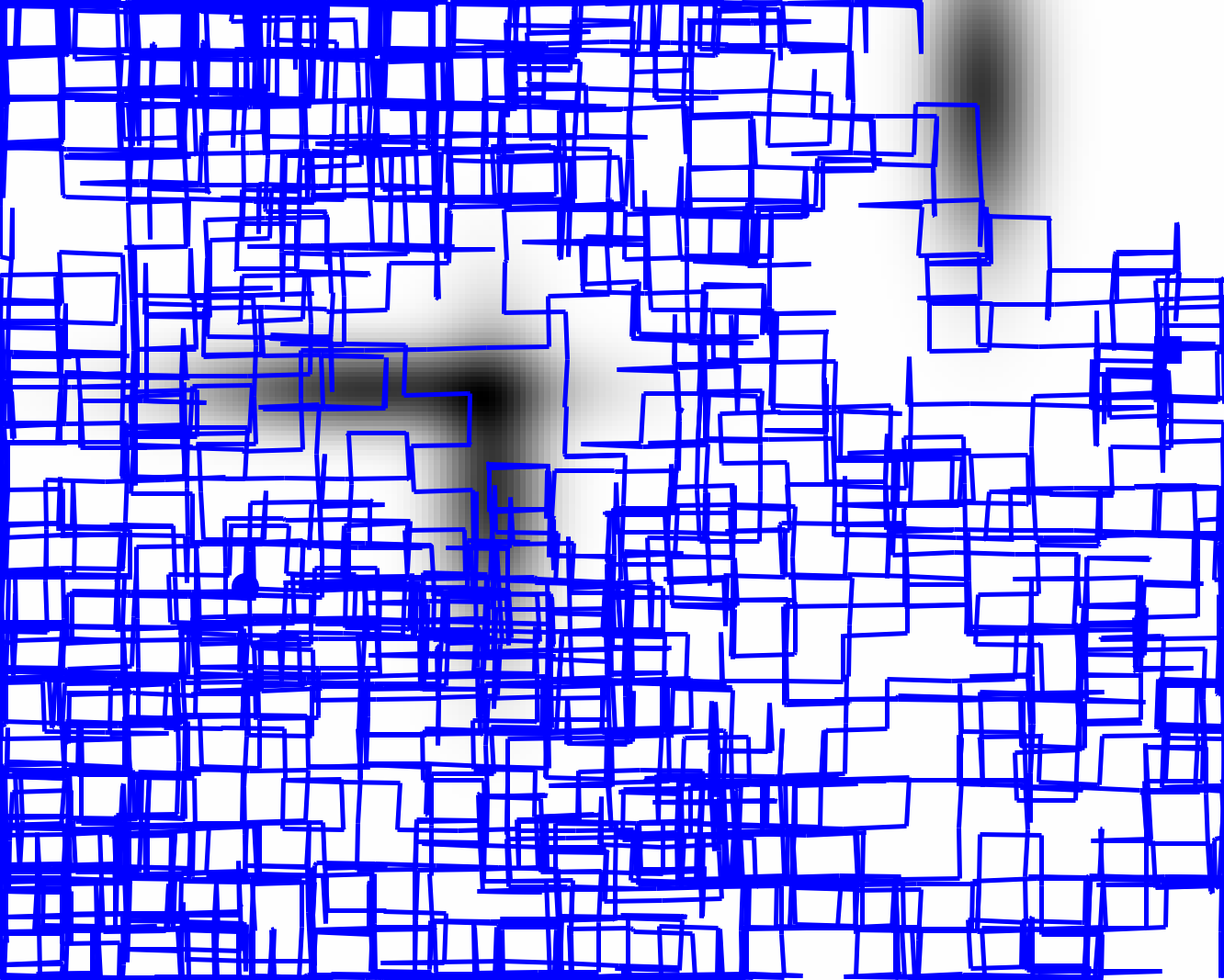} & \gridworld{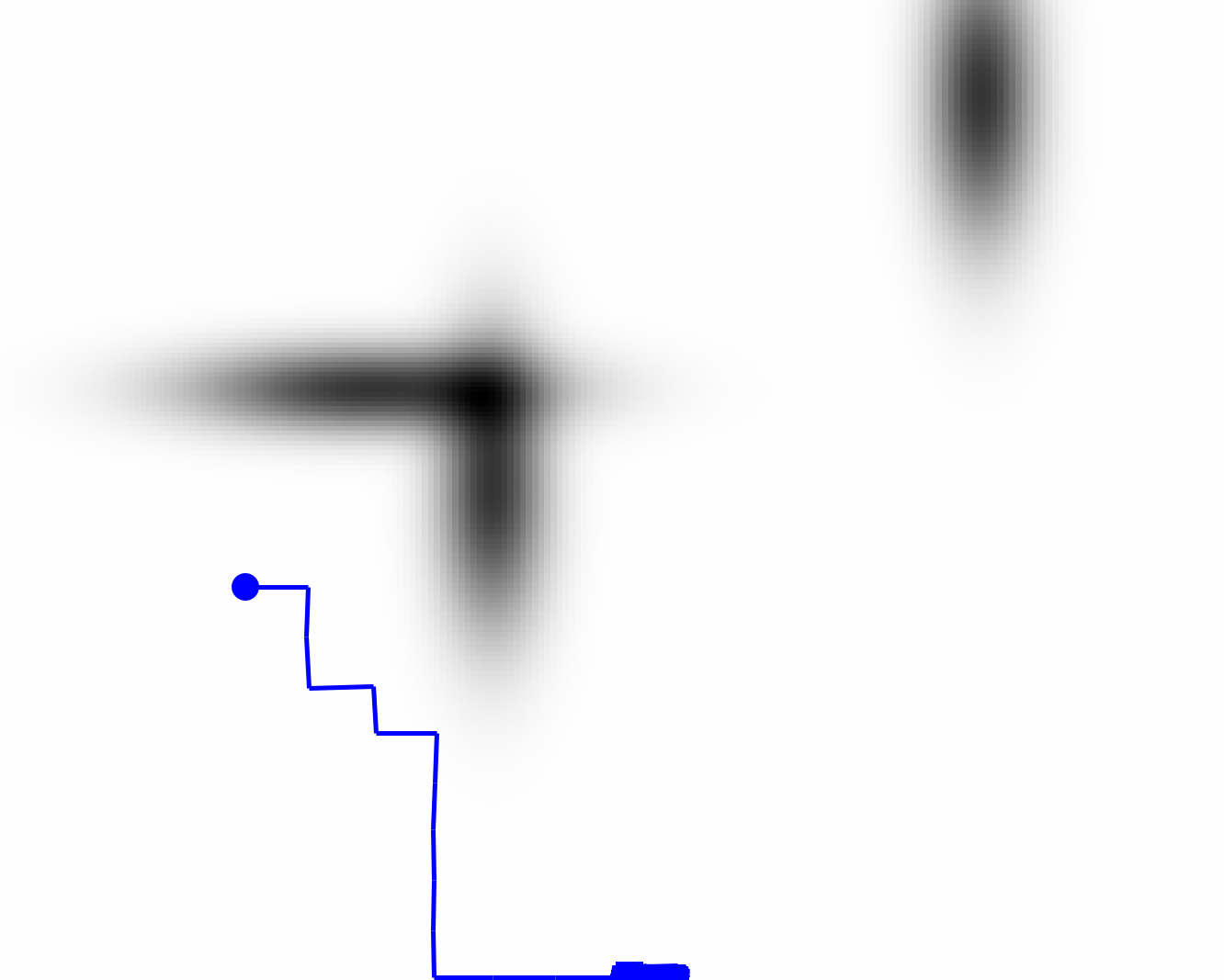} \\
Softmax-GQ & \opac \\
\gridworld{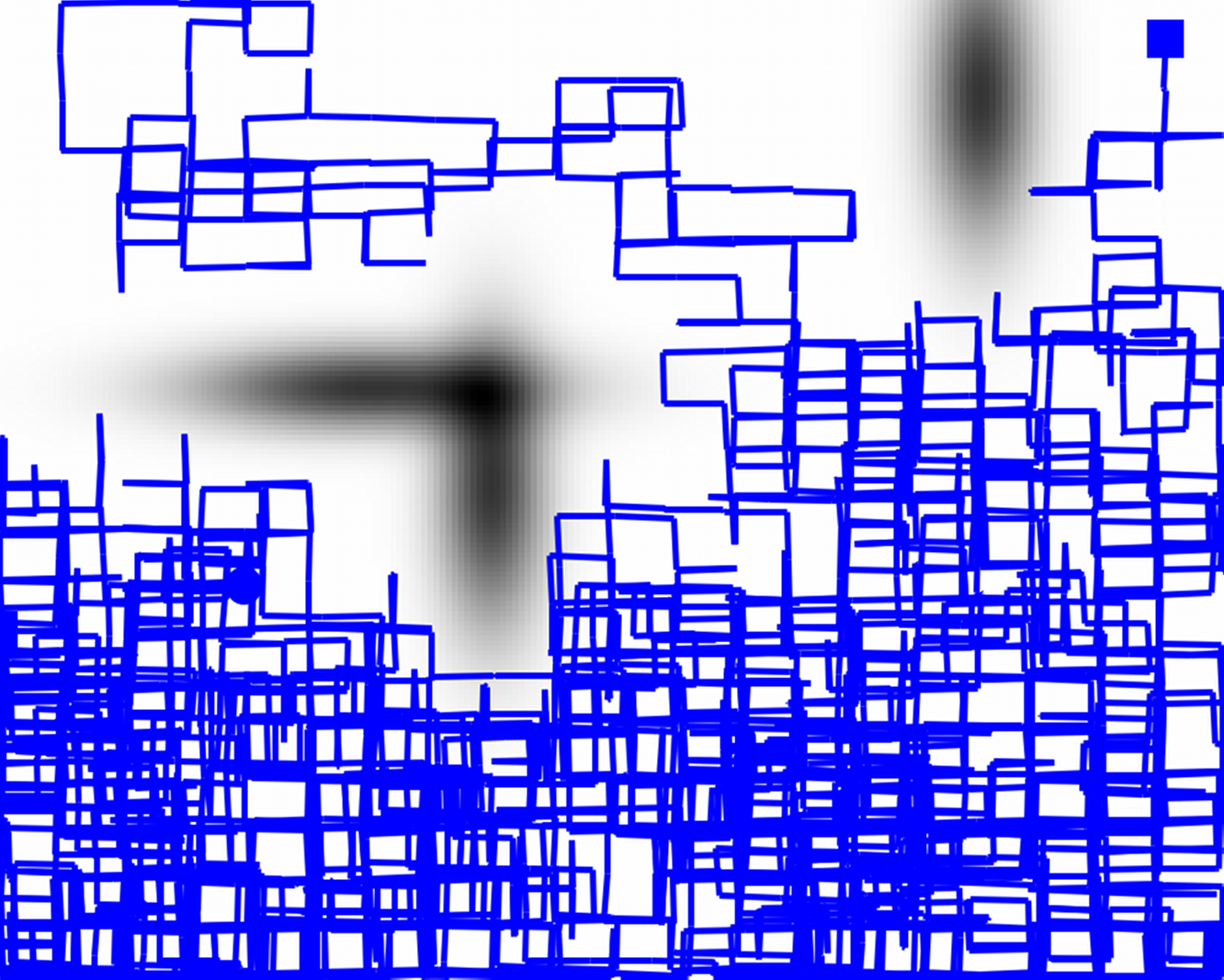} & \gridworld{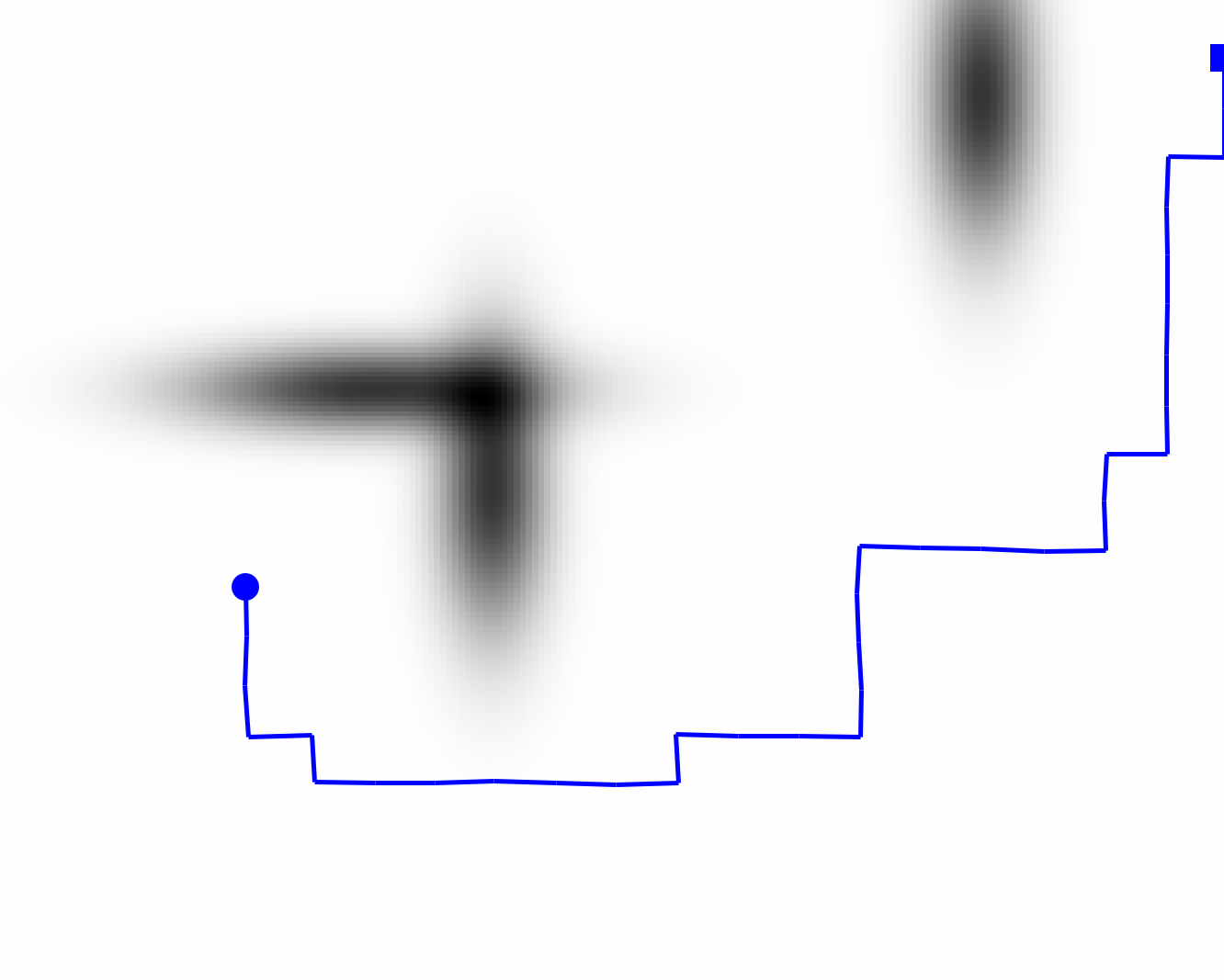} \\
\end{tabular}
\caption{Example of one trajectory for each algorithm in the continuous 2D grid world environment after 5,000 learning episodes from the behavior policy. \opac is the only algorithm that learned to reach the goal reliably.}
\label{fig:discretecontinusgridworldenv}
\end{figure}

\begin{figure*}[t]
\includegraphics[width=.33\linewidth]{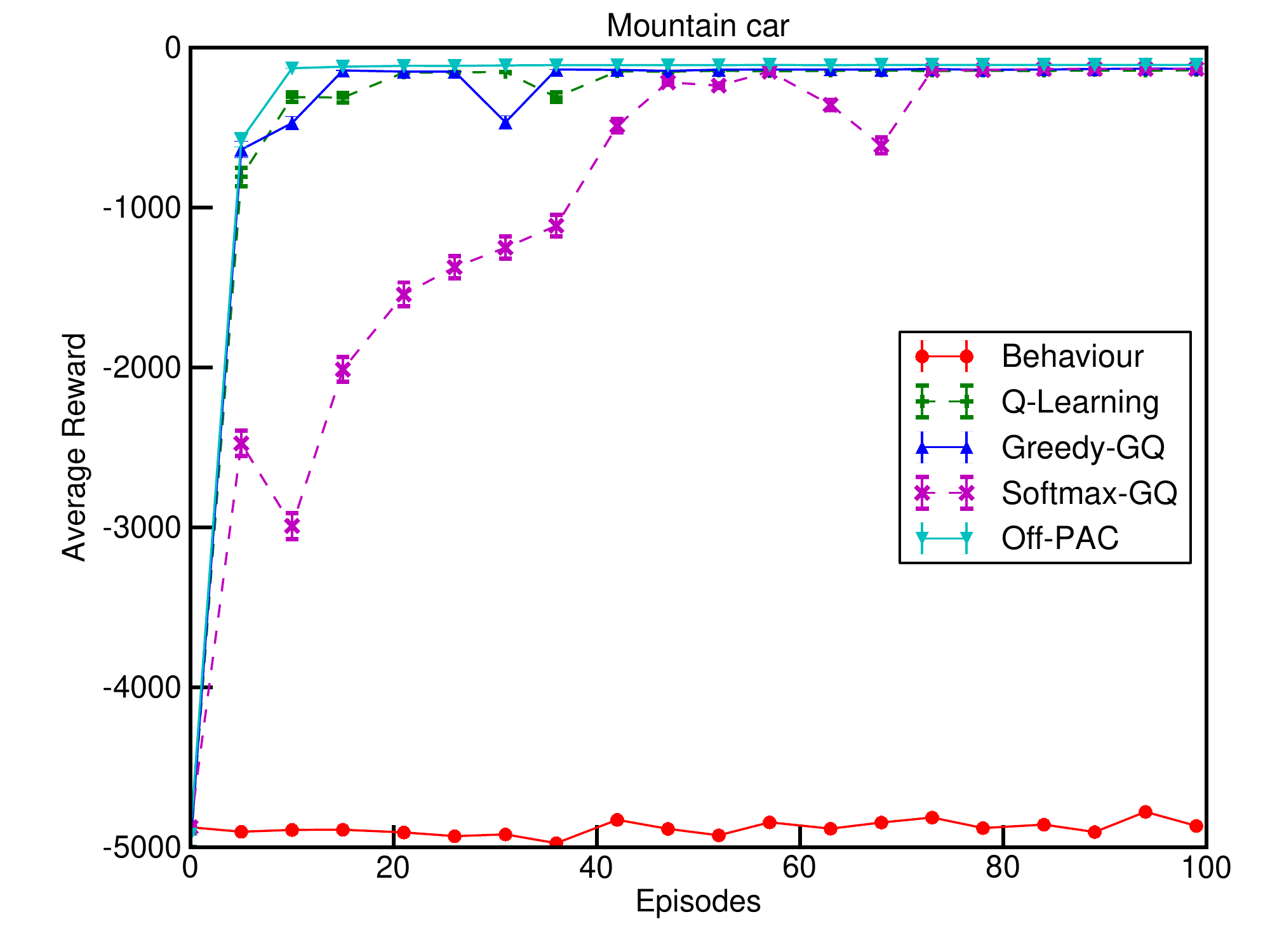}
\includegraphics[width=.33\linewidth]{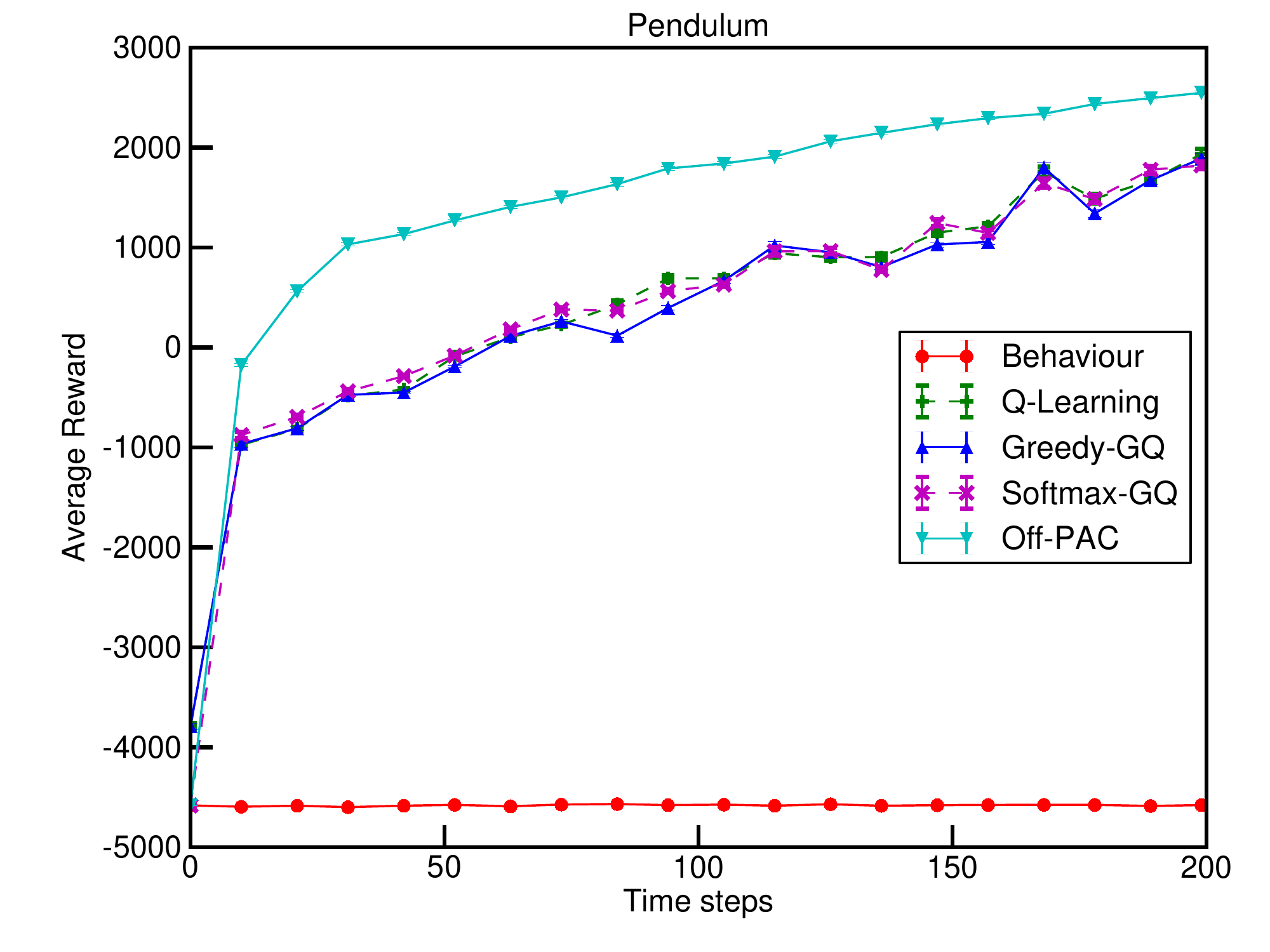}
\includegraphics[width=.33\linewidth]{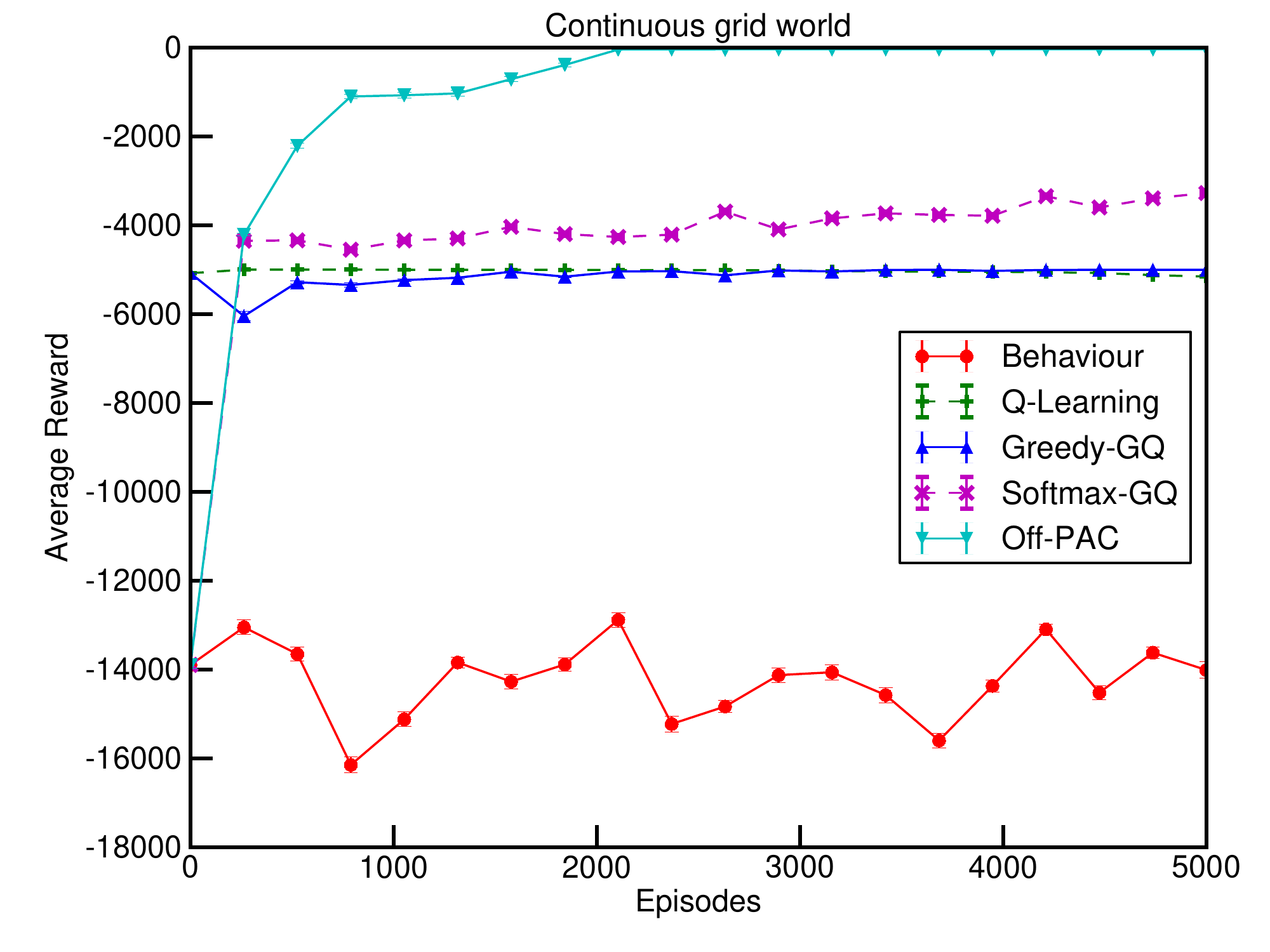}

\centering
\begin{tabular}{|@{\hspace{1pt}}l@{}l@{\hspace{1pt}}||\cl|\cl|\cl|\cl|\cl||\cl|\cl|\cl|\cl|\cl||\cl|\cl|\cl|\cl|\cl|@{}}
\multicolumn{2}{c}{} & \multicolumn{5}{c}{Mountain car}& \multicolumn{5}{c}{Pendulum}& \multicolumn{5}{c}{Continuous grid world}\\ 
\cline{3-17}
\multicolumn{2}{c|}{} & 
                                        $\alpha_w$  & $\alpha_u,\tau$ & $\alpha_v$ & $\lambda$  & \small Reward
                                      & $\alpha_w$  & $\alpha_u,\tau$ & $\alpha_v$ & $\lambda$  & \small Reward
                                      & $\alpha_w$  & $\alpha_u,\tau$ & $\alpha_v$ & $\lambda$  & \small Reward \\ 
\hline
\algoname{Behavior}:    & \asympt    & \sm{\na}    & \sm{\na}        & \sm{\na}   & \sm{\na}   & \sm{-4822\pm6}   
                                      & \sm{\na}    & \sm{\na}        & \sm{\na}   & \sm{\na}   & \sm{-4582\pm0} 
                                      & \sm{\na}    & \sm{\na}        & \sm{\na}   & \sm{\na}   & \sm{-13814\pm127} \\
                         & \overa     & \sm{\na}    & \sm{\na}        & \sm{\na}   & \sm{\na}   & \sm{-4880\pm2}   
                                      & \sm{\na}    & \sm{\na}        & \sm{\na}   & \sm{\na}   & \sm{-4580\pm.3} 
                                      & \sm{\na}    & \sm{\na}        & \sm{\na}   & \sm{\na}   & \sm{-14237\pm33} \\
\hline
\algoname{Q($\lambda$)}: & \asympt    & \sm{\na}    & \sm{\na}        & \sm{.1}    & \sm{.6}    & \sm{-143\pm.4}  
                                      & \sm{\na}    & \sm{\na}        & \sm{.5}    & \sm{.99}   & \sm{1802\pm35}
                                      & \sm{\na}    & \sm{\na}        & \sm{.0001} & \sm{0}     & \sm{-5138\pm.4} \\
                         & \overa     & \sm{\na}    & \sm{\na}        & \sm{.1}    & \sm{0}     & \sm{-442\pm4}  
                                      & \sm{\na}    & \sm{\na}        & \sm{.5}    & \sm{.99}   & \sm{376\pm15}
                                      & \sm{\na}    & \sm{\na}        & \sm{.0001} & \sm{0}     & \sm{-5034\pm.2} \\
\hline
\algoname{Greedy-GQ}:    & \asympt    & \sm{.0001}  & \sm{\na}        & \sm{.1}    & \sm{.4}    & \sm{-131.9\pm.4} 
                                      & \sm{0}      & \sm{\na}        & \sm{.5}    & \sm{.4}    & \sm{1782\pm31} 
                                      & \sm{0.05}   & \sm{\na}        & \sm{1.0}   & \sm{.2}    & \sm{-5002\pm.2} \\
                         & \overa     & \sm{.0001}  & \sm{\na}        & \sm{.1}    & \sm{.2}    & \sm{-434\pm4} 
                                      & \sm{.0001}  & \sm{\na}        & \sm{.01}   & \sm{.4}    & \sm{785\pm11} 
                                      & \sm{0}      & \sm{\na}        & \sm{.0001} & \sm{0}     & \sm{-5034\pm.2} \\
\hline
\algoname{Softmax-GQ}:   & \asympt    & \sm{.0005}  & \sm{.1}         & \sm{.1}    & \sm{.4}    & \sm{-133.4\pm.4} 
                                      & \sm{0}      & \sm{.1}         & \sm{.5}   & \sm{.4}     & \sm{1789\pm32}
                                      & \sm{.1}     & \sm{50}         & \sm{.5}    & \sm{.6}    & \sm{-3332\pm20} \\
                         & \overa     & \sm{.0001}  & \sm{.1}         & \sm{.05}   & \sm{.2}    & \sm{-470\pm7} 
                                      & \sm{.0001}  & \sm{.05}        & \sm{.005}  & \sm{.6}     & \sm{620\pm11}
                                      & \sm{.1}     & \sm{50}         & \sm{.5}    & \sm{.6}    & \sm{-4450\pm11} \\
\hline
\algoname{\opac}:         & \asympt    & \sm{.0001}  & \sm{1.0}        & \sm{.05}   & \sm{0}     & \sm{-108.6\pm.04}
                                      & \sm{.005}   & \sm{.5}         & \sm{.5}    & \sm{0}     & \sm{2521\pm17}
                                      & \sm{0}      & \sm{.001}       & \sm{.1}    & \sm{.4}    & \sm{-37\pm.01}\\
                         & \overa     & \sm{.001}   & \sm{1.0}        & \sm{.5}    & \sm{0}     & \sm{-356\pm.4}
                                      & \sm{0}      & \sm{.5}         & \sm{.5}    & \sm{0}     & \sm{1432\pm10}
                                      & \sm{0}      & \sm{.001}       & \sm{.005}  & \sm{.6}    & \sm{-1003\pm6}\\
\hline
\end{tabular}
\caption{Performance of \opac compared to the performance of Q($\lambda$), Greedy-GQ, and Softmax-GQ when learning off-policy from a random behavior policy. Final performance selected the parameters for the best performance for the last 10\% of the run, whereas the overall performance was over all the runs. The plots on the top show the learning curve for the best parameters for the final performance. \opac had always the best performance and was the only algorithm able to learn to reach the goal reliably in the continuous grid world. Performance is indicated with the standard error.}
\label{fig:results}
\end{figure*}

The feature vectors $\x_s$ were binary vectors constructed according to the standard tile-coding technique (Sutton \& Barto, 1998). For all problems, we used ten tilings, each of roughly $10\times10$ over the joint space of the two state variables, then hashed to a vector of dimension $10^6$. An addition feature was added that was always 1. State-action features, $\comp_{s,a}$, were also $10^6 + 1$ dimensional vectors constructed by also hashing the actions. We used a constant $\gamma=0.99$. All the weight vectors were initialized to 0. We performed a parameter sweep to select the following parameters: 1) the step size $\alpha_v$ for Q($\lambda$), 2) the step-sizes $\alpha_v$ and $\alpha_w$ for the two vectors in Greedy-GQ, 3) $\alpha_v$, $\alpha_w$ and the temperature $\tau$ of the target policy distribution for Softmax-GQ and 4) the step sizes $\alpha_v$, $\alpha_w$ and $\alpha_u$ for \opac. For the step sizes, the sweep was done over the following values: $\{ 10^{-4}, 5 \cdot 10^{-4}, 10^{-3}, \ldots, .5, 1. \}$ divided by 10+1=11, that is the number of tilings plus~1. To compare TD methods to gradient-TD methods, we also used $\alpha_w=0$. The temperature parameter, $\tau$, was chosen from $\{.01, .05, .1, .5, 1, 5, 10, 50, 100\}$ and $\lambda$ from $\{ 0, .2, .4, .6, .8, .99 \}$. We ran thirty runs with each setting of the parameters.

For each parameter combination, the learning algorithm updates a target policy online from the data generated by the behavior policy. For all the problems, the target policy was evaluated at 20 points in time during the run by running it 5 times on another instance of the problem. The target policy was \emph{not} updated during evaluation, ensuring that it was learned only with data from the behavior policy. 

Figure~\ref{fig:results} shows results on three problems. Softmax-GQ and \opac improved their policy compared to the behavior policy on all problems, while the improvements for Q($\lambda$) and Greedy-GQ is limited on the continuous grid world. \opac performed best on all problems. On the continuous grid world, \opac was the only algorithm able to learn a policy that reliably found the goal after 5,000 episodes (see Figure~\ref{fig:discretecontinusgridworldenv}). On all problems, \opac had the lowest standard error.

\section{Discussion} \label{sec:discussion}


\opac, like other two-timescale update algorithms, can be sensitive to parameter choices, particularly the step-sizes. \opac has four parameters: $\lambda$ and the three step sizes, $\alpha_v$ and $\alpha_w$ for the critic and $\alpha_u$ for the actor. In practice, the following procedure can be used to set these parameters. The value of $\lambda$, as with other algorithms, will depend on the problem and it is often better to start with low values (less than .4). A common heuristic is to set $\alpha_v$ to $0.1$ divided by the norm of the feature vector, $\x_s$, while keeping the value of $\alpha_w$ low. Once GTD($\lambda$) is stable learning the value function with $\alpha_u=0$, $\alpha_u$ can be increased so that the policy of the actor can be improved. This corroborates the requirements in the proof, where the step-sizes should be chosen so that the slow update (the actor) is not changing as quickly as the fast inner update to the value function weights (the critic). 

As mentioned by Borkar (2008, p.~75), another scheme that works well in practice is to use the restrictions on the step-sizes in the proof and to also subsample updates for the slow update. Subsampling updates means only updating every $\{tN, t \ge 0\}$, for some $N > 1$: the actor is fixed in-between $tN$ and $(t+1)N$ while the critic is being updated. This further slows the actor update and enables an improved value function estimate for the current policy, $\pi$. 

In this work, we did not explore incremental \textit{natural} actor-critic methods (Bhatnagar et al., 2009), which use the \textit{natural gradient} as opposed to the conventional gradient. The extension to off-policy natural actor-critic should be straightforward, involving only a small modification to the update and analysis of this new dynamical system (which will have similar properties to the original update).


Finally, as pointed out by Precup et al.~(2006), off-policy updates can be more noisy compared to on-policy learning. The results in this paper suggest that \opac is more robust to such noise because it has lower variance than the action-value based methods. Consequently, we think \opac is a promising direction for extending off-policy learning to a more general setting such as continuous action spaces.

\section{Conclusion}

This paper proposed a new algorithm for learning control off-policy, called \opac (Off-Policy Actor-Critic). We proved that \opac converges in a standard off-policy setting. We provided one of the first empirical evaluations of off-policy control with the new gradient-TD methods and showed that \opac has the best final performance  on three benchmark problems and consistently has the lowest standard error. Overall, \opac is a significant step toward robust off-policy control. 


\section{Acknowledgments}

This work was supported by MPrime, the Alberta Innovates Centre for Machine Learning, the Glenrose Rehabilitation Hospital Foundation, Alberta Innovates---Technology Futures, NSERC and the ANR MACSi project. Computational time was provided by Westgrid and the M\'{e}socentre de Calcul Intensif Aquitain. 


{\small
\subsection*{Appendix: See {\footnotesize http://arXiv.org/abs/1205.4839}}

\subsection*{References}

\parskip=2pt
\parindent=0pt
\def\hangin{\hangindent=0.15in}

\def\bibitem[#1]#2{\hangin}


\bibitem[Baird(1995)]{baird1995residual}
Baird, L. (1995).
\newblock Residual algorithms: Reinforcement learning with function
  approximation.
\newblock In \emph{Proceedings of the Twelfth International Conference on
  Machine Learning}, pp.\  30--37. Morgan Kaufmann.

\bibitem[Bhatnagar et~al.(2009)Bhatnagar, Sutton, Ghavamzadeh, and
  Lee]{bhatnagar2009natural}
Bhatnagar, S., Sutton, R.~S., Ghavamzadeh, M., Lee, M. (2009).
\newblock Natural actor-critic algorithms.
\newblock \emph{Automatica 45}(11):2471--2482.

\bibitem[Borkar(2008)]{borkar2008stochastic}
Borkar, V.~S. (2008).
\newblock \emph{Stochastic approximation: A dynamical systems viewpoint}.
\newblock Cambridge Univ Press.

\hangin
Bradtke, S.~J., Barto, A.~G. (1996).
Linear least-squares algorithms for temporal difference learning.
\emph{Machine Learning 22}:33--57.

\bibitem[Delp(2010)]{delp2010thesis}
Delp, M. (2010).
\newblock Experiments in off-policy reinforcement learning with the
  $\mbox{GQ}(\lambda)$ algorithm.
\newblock Masters thesis, University of Alberta.

\bibitem[Doya(2000)]{doya2000reinforcement}
Doya, K. (2000).
\newblock Reinforcement learning in continuous time and space.
\newblock \emph{Neural computation 12}:219--245.


\hangin
Lagoudakis, M., Parr, R. (2003).
Least squares policy iteration.
\textit{Journal of Machine Learning Research 4}:1107--1149.

\bibitem[Maei \& Sutton(2010)Maei and Sutton]{Maei-agi-10}
Maei, H.~R., Sutton, R.~S. (2010).
\newblock {GQ}($\lambda$): A general gradient algorithm for temporal-difference
  prediction learning with eligibility traces.
\newblock In \emph{Proceedings of the Third Conf. on Artificial General
  Intelligence}.

\bibitem[Maei(2011)]{maei2011gradient}
Maei, H.~R. (2011).
\newblock \emph{Gradient Temporal-Difference Learning Algorithms}.
\newblock PhD thesis, University of Alberta.

\bibitem[Maei et~al.(2009)Maei, Szepesv{\'a}ri, Bhatnagar, Precup, Silver, and
  Sutton]{maei2009nonlinear}
Maei, H.~R., Szepesv{\'a}ri, C., Bhatnagar, S., Precup, D., Silver, D., 
  Sutton, R.~S. (2009).
\newblock Convergent temporal-difference learning with arbitrary smooth
  function approximation.
\newblock \emph{Advances in Neural Information Processing Systems
  22}:1204--1212.

\bibitem[Maei et~al.(2010)Maei, Szepesv{\'a}ri, Bhatnagar, and
  Sutton]{maei2010toward}
Maei, H.~R., Szepesv{\'a}ri, C., Bhatnagar, S., Sutton, R.~S. (2010).
\newblock Toward off-policy learning control with function approximation.
\newblock \emph{Proceedings of the 27th International Conference on Machine
  Learning}.
  
\hangin
Marbach, P., Tsitsiklis, J.~N. (1998).
Simulation-based optimization of Markov reward processes.
Technical report LIDS-P-2411.

\bibitem[Pemantle(1990)]{pemantle1990nonconvergence}
Pemantle, R. (1990).
\newblock Nonconvergence to unstable points in urn models and stochastic
  approximations.
\newblock \emph{The Annals of Probability 18}(2):698--712.

\bibitem[Peters \& Schaal(2008)Peters and Schaal]{peters2008natural}
Peters, J., Schaal, S. (2008).
\newblock Natural actor-critic.
\newblock \emph{Neurocomputing 71}(7):1180--1190.

\bibitem[Precup et~al.(2006)Precup, Sutton, Paduraru, Koop, and
  Singh]{precup2006off}
Precup, D., Sutton, R.S., Paduraru, C., Koop, A., Singh, S. (2006).
\newblock Off-policy learning with recognizers.
\newblock \emph{Neural Information Processing Systems 18}.

\bibitem[Smart \& Pack~Kaelbling(2002)Smart and
  Pack~Kaelbling]{smart2002effective}
Smart, W.D., Pack~Kaelbling, L. (2002).
\newblock Effective reinforcement learning for mobile robots.
\newblock In \emph{Proceedings of International Conference on Robotics and Automation}, volume~4, pp.\  3404--3410.

\bibitem[Sutton \& Barto(1998)Sutton and Barto]{Sutton-Barto-98}
Sutton, R.~S., Barto, A.~G. (1998).
\newblock \emph{Reinforcement Learning: An Introduction}.
\newblock {MIT} Press.

\bibitem[Sutton et~al.(2000)Sutton, McAllester, Singh, and
  Mansour]{sutton2000policy}
Sutton, R.~S., McAllester, D., Singh, S., Mansour, Y. (2000).
\newblock Policy gradient methods for reinforcement learning with function
  approximation.
\newblock \emph{Advances in Neural Information Processing Systems 12}.

\bibitem[Sutton et~al.(2008)Sutton, Szepesv{\'a}ri, and
  Maei]{sutton2009convergent}
Sutton, R.~S., Szepesv{\'a}ri, {Cs}., Maei, H.~R. (2008).
\newblock A convergent ${O}(n)$ algorithm for off-policy temporal-difference
  learning with linear function approximation.
\newblock In \emph{Advances in Neural Information Processing Systems 21}, pp.\
  1609--1616.

\bibitem[Sutton et~al.(2009)Sutton, Maei, Precup, Bhatnagar, Silver,
  Szepesv{\'a}ri, and Wiewiora]{sutton2009fast}
Sutton, R.~S., Maei, H.~R., Precup, D., Bhatnagar, S., Silver, D.,
  Szepesv{\'a}ri, Cs., Wiewiora, E. (2009).
\newblock Fast gradient-descent methods for temporal-difference learning with
  linear function approximation.
\newblock In \emph{Proceedings of the 26th Annual International Conference on
  Machine Learning}, pp.\  993--1000.

\bibitem[Sutton et~al.(2011)Sutton, Modayil, Delp, Degris, Pilarski, and
  Precup]{Sutton-aamas-11}
Sutton, R.~S., Modayil, J., Delp, M., Degris, T., Pilarski, P.~M., and Precup, D. (2011).
\newblock Horde: A scalable real-time architecture for learning knowledge from
  unsupervised sensorimotor interaction.
\newblock In \emph{Proceedings of the 10th International Conference on
  Autonomous Agents and Multiagent Systems}.

\bibitem[Watkins \& Dayan(1992)Watkins and Dayan]{watkins1992q}
Watkins, C.~J.~C.~H., Dayan, P. (1992).
\newblock Q-learning.
\newblock \emph{Machine Learning 8}(3):279--292.
\newpage
}

\onecolumn
\appendix

\section{Appendix of Off-Policy Actor-Critic}

\subsection{Policy Improvement and Policy Gradient Theorems}

\begin{thmnocount}{\ref{thm:policyimprovement}}[Off-Policy Policy Improvement Theorem]\\
Given any policy parameter $\ap$, let 
\[
\ap' = \ap + \alpha \, \gradtobj
\]
Then there exists an $\epsilon>0$ such that, for all positive $\alpha<\epsilon$, 
\begin{align*}
\obj(\ap') \ge \obj(\ap)
\end{align*}
Further, if $\pi$ has a tabular representation (i.e., separate weights for each state), then
$V^{\pi_{\ap'},\gamma}(s) \ge V^{\pi_{\ap},\gamma}(s)$ for all $s \in \St$.
\end{thmnocount}

\begin{proof} 
Notice first that for any $(s,a)$, the gradient $\grad{\pi(a|s)}{\ap}$ is the direction to increase the probability of action~$a$ according to function $\pi(\cdot | s)$. For an appropriate step size $\alpha_{u,t}$ (so that the update to $\pi_{\ap'}$ increases the objective with the action-value function $Q^{\pi_{\ap},\gamma}$, fixed as the old action-value function), we can guarantee that 
\begin{align*} 
\obj(\ap) &= \sum_{s \in \St} d^{b}(s) \sum_{a \in \Ac} \pi_{\ap}(a | s) Q^{\pi_{\ap},\gamma}(s,a)\\
&\le \sum_{s \in \St} d^{b}(s) \sum_{a \in \Ac}\pi_{\ap'} (a | s) Q^{\pi_{\ap},\gamma}(s,a)
\end{align*}
Now we can proceed similarly to the Policy Improvement theorem proof provided by Sutton and Barto (1998) by extending the right-hand side using the definition of $Q^{\pi,\gamma}(s,a)$ (equation~\ref{eq:actionvalue}):
\begin{align*}
\obj(\ap_t) 
&\leq \sum_{s \in \St} d^b(s) \sum_{a \in \Ac} \pi_{\ap'}(a|s) \E{\pi_{\ap'},\gamma}{r_{t+1} + \gamma_{t+1} V^{\pi_{\ap},\gamma}(s_{t+1})}\\
& \le \sum_{s \in \St} d^b(s) \sum_{a \in \Ac} \pi_{\ap'}(a|s) \E{\pi_{\ap'},\gamma}{r_{t+1} + \gamma_{t+1} r_{t+2} + \gamma_{t+2} V^{\pi_{\ap},\gamma}(s_{t+2})}\\
& \vdots \\
& \leq \sum_{s \in \St} d^b(s) \sum_{a \in \Ac} \pi_{\ap'}(a|s) Q^{\pi_{\ap'},\gamma}(s,a)\\
& = \obj(\ap')
\end{align*}
The second part of the Theorem has similar proof to the above. With a tabular representation for $\pi$, we know that the gradient satisfies:
$$ 
\sum_{a \in \Ac} \pi_{\ap}(a|s) Q^{\pi_{\ap}, \gamma}(s,a)  \le \sum_{a \in \Ac} \pi_{\ap'}(a|s) Q^{\pi_{\ap}, \gamma}(s,a) 
$$
because the probabilities can be updated independently for each state with separate weights for each state. 

Now for any $s \in \St$:
\begin{align*}
V^{\pi_{\ap}, \gamma}(s) &= \sum_{a \in \Ac} \pi_{\ap}(a|s) Q^{\pi_{\ap}, \gamma}(s,a)\\
&\leq \sum_{a \in \Ac} \pi_{\ap'}(a|s) Q^{\pi_{\ap}, \gamma}(s,a)\\
&\leq \sum_{a \in \Ac} \pi_{\ap'}(a|s) \E{\pi_{\ap'},\gamma}{r_{t+1} + \gamma_{t+1} V^{\pi_{\ap},\gamma}(s_{t+1})}\\
& \le \sum_{a \in \Ac} \pi_{\ap'}(a|s) \E{\pi_{\ap'},\gamma}{r_{t+1} + \gamma_{t+1} r_{t+2} + \gamma_{t+2} V^{\pi_{\ap},\gamma}(s_{t+2})}\\
& \vdots \\
& \leq \sum_{a \in \Ac} \pi_{\ap'}(a|s) Q^{\pi_{\ap'},\gamma}(s,a)\\
& = V^{\pi_{\ap'}, \gamma}(s)
\end{align*}
\vspace*{-20pt}
\end{proof} 

\begin{thmnocount}{\ref{thm:policygradient}} [Off-Policy Policy Gradient Theorem]\\
Let $\conv = \{\ap \in \UU \ | \ \gradobj = 0\}$ and $\conva = \{\ap \in \UU \ | \ \gradtobj = 0\}$, which are both non-empty by Assumption \ref{req:projection}. If the value function can be represented by our function class, then 
$$ \conv \subset \conva$$ 
Moreover, if we use a tabular representation for $\pi$, then
$$ \conv = \conva$$
\end{thmnocount}
\begin{proof}
This theorem follows from our policy improvement theorem. 

Assume there exists $\ap^* \in \conv$ such that $\ap^* \notin \conva$. Then $\grad{\obj(\ap)}{\ap^*} = 0$ but $\gradt{\obj(\ap)}{\ap^*} \neq 0$. By the policy improvement theorem (Theorem \ref{thm:policyimprovement}), we know that $\obj(\ap^* + \alpha_{u,t}\gradt{\obj(\ap)}{\ap^*}) > \obj(\ap)$, for some positive $\alpha_{u,t}$. However, this is a contradiction, as the true gradient is zero. Therefore, such an $\ap^*$ cannot exist.

For the second part of the theorem, we have a tabular representation, in other words, each weight corresponds to exactly one state. Without loss of generality, assume each state $s$ is represented with $m \in \Natural$ weights, indexed by let $i_{s,1} \ldots i_{s,m}$ in the vector $\ap$. Therefore, for any state, $s$
\begin{align*}
\sum_{s' \in \St} d^b(s') \sum_{a \in \Ac} \partless{}{\ap_{i_{s,j}}} \pi_\ap(a|s') Q^{\pi_\ap,\gamma}(s',a) = d^b(s) \sum_{a \in \Ac} \partless{}{\ap_{i_{s,j}}} \pi_\ap(a|s) Q^{\pi_\ap,\gamma}(s,a)
\doteq {\bf g}_1(\ap_{i_{s,j}})
\end{align*}
Assume there exists $s \in \St$ such that ${\bf g}_1(\ap_{i_{s,j}}) = 0$ $\forall j$ but there exists $1 \le k \le m$ for ${\bf g}_2(\ap_{i_{s,k}}) \doteq \sum_{s' \in \St} d^b(s') \sum_{a \in \Ac} \pi_\ap(a|s') \partless{}{\ap_{i_{s,k}}} Q^{\pi_\ap,\gamma}(s',a)$ such that ${\bf g}_2(\ap_{i_{s,k}}) \neq 0$.
$\partless{}{\ap_{i_s}} Q^{\pi_\ap,\gamma}(s',a)$ can only increase the value of $Q^{\pi_\ap,\gamma}(s,a)$ locally (i.e., shift the probabilities of the actions to increase return), because it cannot change the value in other states ($\ap_{i_s}$ is only used for state $s$ and the remaining weights are fixed when this partial derivative is computed). 
Therefore, since ${\bf g}_2(\ap_{i_{s,k}}) \neq 0$, we must be able to increase the value of state $s$ by changing the probabilities of the actions in state $s$
$$ \implies  \sum_{j = 1}^m \sum_{a \in \Ac} \partless{}{\ap_{i_{s,j}}} \pi_\ap(a|s) Q^{\pi_\ap,\gamma}(s,a) \neq 0 $$
which is a contradiction (since we assumed ${\bf g}_1(\ap_{i_{s,j}}) = 0$ $\forall j$).

Therefore, in the tabular case, whenever $\sum_s d^b(s) \sum_a \nabla_\ap \pi_\ap(a|s) Q^{\pi_\ap,\gamma}(s,a) = 0$, then $\sum_s d^b(s) \sum_a \pi_\ap(a|s)  \nabla_\ap Q^{\pi_\ap,\gamma}(s,a) = 0$, implying that $\conva \subset \conv$. Since we already know that $\conv \subset \conva$, then we can conclude that for a tabular representation for $\pi$, $\conv = \conva$.
\end{proof}

%


\newcommand{\Esa}[1]{\Eb{s_t=s, a_t=a}{#1}}
\newcommand{\Esas}[1]{\Eb{s_t, a_t, s_{t+1}}{#1}}
\newcommand{\tde}{\delta}
 \allowdisplaybreaks[1]

\subsection{Forward/Backward view analysis}
\label{sec:forwardbackwardanalysis}

In this section, we prove the key relationship between the forward and backward views:
\setcounter{equation}{5}
\begin{equation}
\label{eq:forwardbackward}
\Eb{}{\rho(s_t,a_t) \comp(s_t,a_t) \left(\Gret_t - \vfa(s_t)\right)} = \Eb{}{\delta_t \e_{t}} 
\end{equation}
where, in these expectations, and in all the expectations in this section, the random variables (indexed by time step) are from their stationary distributions under the behavior policy. We assume that the behavior policy is stationary and that the Markov chain is aperiodic and irreducible (i.e., that we have reached the limiting distribution, $d^b$, over $s \in \St$). Note that under these definitions:
\begin{equation}
\label{eq:shifts}
\Eb{}{X_t} = \Eb{}{X_{t+k}}
\end{equation}
for all integer $k$ and for all random variables $X_t$ and $X_{t+k}$ that are simple temporal shifts of each other.
To simplify the notation in this section, we define $\rho_t=\rho(s_t,a_t)$, $\comp_t = \comp(s_t,a_t)$, $\gamma_t = \gamma(s_t)$, and $\deltaret_t=\Gret_t - \vfa(s_t)$.

\begin{proof}
First we note that $\deltaret_t$, which might be called the forward-view TD error, can be written recursively:
\begin{align}
\label{eq:deltarec}
\deltaret_t &= \Gret_t - \vfa(s_t) \nonumber\\
&=  r_{t+1} + (1 - \lambda) \gamma_{t+1} \vfa(s_{t+1}) + \lambda \gamma_{t+1} \rho_{t+1} \Gret_{t+1} - \vfa(s_t) \nonumber\\
&=  r_{t+1} + \gamma_{t+1} \vfa(s_{t+1}) - \lambda \gamma_{t+1} \vfa(s_{t+1}) + \lambda \gamma_{t+1} \rho_{t+1} \Gret_{t+1} - \vfa(s_t)  \nonumber\\
&=  r_{t+1} + \gamma_{t+1} \vfa(s_{t+1}) - \vfa(s_t) + \lambda \gamma_{t+1} \left( \rho_{t+1} \Gret_{t+1} - \vfa(s_{t+1}) \right) \nonumber\\
&=  \tde_t + \lambda \gamma_{t+1} \left( \rho_{t+1} \Gret_{t+1} - \rho_{t+1} \vfa(s_{t+1}) - (1 - \rho_{t+1}) \vfa(s_{t+1}) \right)  \nonumber\\
&=  \tde_t + \lambda \gamma_{t+1} \left( \rho_{t+1} \deltaret_{t+1} - (1 - \rho_{t+1}) \vfa(s_{t+1}) \right) 
\end{align}
where $\tde_t=r_{t+1} + \gamma_{t+1} \vfa(s_{t+1}) - \vfa(s_t)$ is the conventional one-step TD error.

Second, we note that the following expectation is zero:
\begin{align}
\label{eq:Ezero}
&\Eb{}{\rho_t \comp_t \gamma_{t+1} (1-\rho_{t+1}) \vfa(s_{t+1})} \nonumber\\
&\hspace{3cm}= \sum_s d^b(s) \sum_a b(a|s) \rho(s,a) \comp(s,a) \sum_{s'} P(s'|s,a){\gamma(s') \left(1-\sum_{a'}b(a'|s')\rho(s',a')\right) \vfa(s')} \nonumber\\
&\hspace{3cm}= \sum_s d^b(s) \sum_a b(a|s) \rho(s,a) \comp(s,a) \sum_{s'} P(s'|s,a){\gamma(s') \left(1-\sum_{a'}b(a'|s')\frac{\pi(a'|s')}{b(a'|s')}\right) \vfa(s')} \nonumber\\
&\hspace{3cm}= \sum_s d^b(s) \sum_a b(a|s) \rho(s,a) \comp(s,a) \sum_{s'} P(s'|s,a){\gamma(s') \left(1-\sum_{a'}\pi(a'|s')\right) \vfa(s')} \nonumber\\
&\hspace{3cm}= 0
\end{align}

We are now ready to prove Equation~\ref{eq:forwardbackward} simply by repeated unrolling and rewriting of the right-hand side, using Equations \ref{eq:deltarec}, \ref{eq:Ezero}, and \ref{eq:shifts} in sequence, until the pattern becomes clear:
\begin{align*}
&\Eb{}{\rho_t \comp_t \left(\Gret_t - \vfa(s_t)\right)} \\
&= \Eb{}{\rho_t \comp_t \left( \tde_t + \lambda \gamma_{t+1} \left( \rho_{t+1} \deltaret_{t+1} - (1 - \rho_{t+1}) \vfa(s_{t+1}) \right) \right)}
\hspace{4.7cm}\mbox{(using (\ref{eq:deltarec}))}\\ 
&=\Eb{}{\rho_t \comp_t \tde_t} + \Eb{}{\rho_t \comp_t \lambda \gamma_{t+1} \rho_{t+1} \deltaret_{t+1}} - \Eb{}{\rho_t \comp_t \lambda \gamma_{t+1} (1-\rho_{t+1}) \vfa(s_{t+1})} \\
&=\Eb{}{\rho_t \comp_t \tde_t} + \Eb{}{\rho_t \comp_t \lambda \gamma_{t+1} \rho_{t+1} \deltaret_{t+1}}
\hspace{7.6cm}\mbox{(using (\ref{eq:Ezero}))}\\
&= \Eb{}{\rho_t \comp_t \tde_t} + \Eb{}{\rho_{t-1} \comp_{t-1} \lambda \gamma_t \rho_t \deltaret_t}
\hspace{7.9cm}\mbox{(using (\ref{eq:shifts}))}\\
&= \Eb{}{\rho_t \comp_t \tde_t} + \Eb{}{\rho_{t-1} \comp_{t-1} \lambda \gamma_t \rho_t \left( \tde_t + \lambda \gamma_{t+1} \left( \rho_{t+1} \deltaret_{t+1} - (1 - \rho_{t+1}) \vfa(s_{t+1}) \right) \right)}
\hspace{1.2cm}\mbox{(using (\ref{eq:deltarec}))}\\
&= \Eb{}{\rho_t \comp_t \tde_t} + \Eb{}{\rho_{t-1} \comp_{t-1} \lambda \gamma_t \rho_t \tde_t} + \Eb{}{\rho_{t-1} \comp_{t-1} \lambda \gamma_t \rho_t \lambda \gamma_{t+1} \rho_{t+1} \deltaret_{t+1} }
\hspace{2.55cm}\mbox{(using (\ref{eq:Ezero}))}\\
&= \Eb{}{\rho_t \tde_t \left( \comp_t + \lambda \gamma_t \rho_{t-1} \comp_{t-1}\right)} + \Eb{}{\lambda^2 \rho_{t-2} \comp_{t-2} \gamma_{t-1} \rho_{t-1} \gamma_t \rho_t \deltaret_t }
\hspace{3.8cm}\mbox{(using (\ref{eq:shifts}))}\\
&= \Eb{}{\rho_t \tde_t \left( \comp_t + \lambda \gamma_t \rho_{t-1} \comp_{t-1}\right)}+ \Eb{}{\lambda^2 \rho_{t-2} \comp_{t-2} \gamma_{t-1} \rho_{t-1} \gamma_t \rho_t \left( \tde_t + \lambda \gamma_{t+1} \left( \rho_{t+1} \deltaret_{t+1} - (1 - \rho_{t+1}) \vfa(s_{t+1}) \right) \right) } \\ 
&= \Eb{}{\rho_t \tde_t \left( \comp_t + \lambda \gamma_t \rho_{t-1} \comp_{t-1}\right)}+ \Eb{}{\lambda^2 \rho_{t-2} \comp_{t-2} \gamma_{t-1} \rho_{t-1} \gamma_t \rho_t \tde_t} + \Eb{}{\lambda^2 \rho_{t-2} \comp_{t-2} \gamma_{t-1} \rho_{t-1} \gamma_t \rho_t \lambda \gamma_{t+1} \rho_{t+1} \deltaret_{t+1}} \\
&= \Eb{}{\rho_t \tde_t \left( \comp_t + \lambda \gamma_t \rho_{t-1} \left(\comp_{t-1} + \lambda \gamma_{t-1} \rho_{t-2} \comp_{t-2} \right) \right)} + \Eb{}{\lambda^3 \rho_{t-3} \comp_{t-3} \gamma_{t-2} \rho_{t-2} \gamma_{t-1} \rho_{t-1} \gamma_t \rho_t \deltaret_t}\\
&\vdots\\
&= \Eb{}{\rho_t \tde_t \left( \comp_t + \lambda \gamma_t \rho_{t-1} \left(\comp_{t-1} + \lambda \gamma_{t-1} \rho_{t-2} \left( \comp_{t-2} + \lambda \gamma_{t-2} \rho_{t-3} \ldots \right) \right) \right)}\\
&=  \Eb{}{\tde_t \e_{t}}
\end{align*} 
where $\e_{t} = \rho_t \left( \comp_t + \lambda \gamma_t \e_{t-1} \right)$.
\end{proof}
 \allowdisplaybreaks[1]

\newcommand{\vfcn}{\chi}
\newcommand{\Arho}{A_\rho(\ap)}
\newcommand{\Complex}{\mathbb{C}}
\newcommand{\matn}[5]{ \[#1 = \left( \begin{array}{cc}
#2 &  #3\\
#4 & #5 \end{array} \right) \]
}
\newcommand{\mat}[4]{\left( \begin{array}{cc}
#1 &  #2\\
#3 & #4 \end{array} \right) 
}




\subsection{Convergence Proofs}

Our algorithm has the same recursive stochastic form that the two-timescale off-policy value-function algorithms have:
\begin{align*}
u_{t+1} = u_t + \alpha_t(h(u_t,z_t) + M_{t+1})\\
z_{t+1} = z_t + \alpha_t(f(u_t,z_t) + N_{t+1})\\
\end{align*}
where $x \in \Real^d, \ h: \Real^d \ra \Real^d$ is a differentiable functions, $\{\alpha_t \}_{k \ge 0}$ is a positive step-size sequence and $\{ M_t \}_{k \ge 0}$ is a noise sequence. 
Again, following the GTD($\lambda$) and GQ($\lambda$) proofs, we study the behavior of the ordinary differential equation
\begin{align*}
\dot u(t) = h(u(t),z)
\end{align*}
Since we have two updates, one for the actor and one for the critic, and those time updates are not linearly separable, we have to do a two timescale analysis (Borkar, 2008). In order to satisfy the conditions for the two-timescale analysis, we will need the following assumptions on our objective, the features and the step-sizes.  Note that it is difficult to prove the boundedness of the iterates without the projection operator we describe below, though the projection was not necessary during experiments.

\begin{enumerate}[label=(A\arabic*)] 
\item The policy function, $\pi_{(\cdot)}(a | s) : \Real^{N_\ap} \ra [0,1]$, is continuously differentiable in $\ap$, $\forall s \in \St, a \in \Ac$. 
\item The update on $\ap_t$ includes a projection operator, $\Gamma: \Real^{N_\ap} \ra \Real^{N_\ap}$ that projects any $\ap$ to a compact set $\UU = \{ \ap \ | \ q_i(\ap) \le 0, i = 1, \ldots, s\} \subset \Real ^{N_\ap}$, where $q_i(\cdot): \Real^{N_\ap} \ra \Real$ are continuously differentiable functions specifying the constraints of the compact region. For each $\ap$ on the boundary of $\UU$, the gradients of the active $q_i$ are considered to be linearly independent. Assume that the compact region, $\UU$, is large enough to contain at least one local maximum of $\obj$. \label{req:projection}
\item The behavior policy has a minimum positive weight for all actions in every state, in other words, $b(a|s) \ge b_{\text{min}}$ $\forall s \in \St, a \in \Ac$, for some $b_{\text{min}} \in (0,1]$. 
\item The sequence $(\x_t, \x_{t+1}, r_{t+1})_{t \ge 0}$ is i.i.d. and has uniformly bounded second moments.
\item For every $\ap \in \UU$ (the compact region to which $\ap$ is projected), $V^{\pi_\ap,\gamma}: S \ra \Real$ is bounded.\label{req:Alast}
\end{enumerate}



\begin{enumerate}[label=(P\arabic*)] 
\item $||\x_t||_\infty < \infty, \ \forall t$, where $\x_t \in \Real^{\Ncp}$
\item The matrices $C = E[\x_t \transp{\x_t}]$ and $A = E[\x_t(\x_t - \gamma \transp{\x_{t+1})}]$ are non-singular and uniformly bounded. $A$, $C$ and $E[r_{t+1} \x_t]$ are well-defined because the distribution of $(\x_t, \x_{t+1}, r_{t+1})$ does not depend on $t$.
\end{enumerate}

\begin{enumerate}[label=(S\arabic*)] 
\item $\alpha_{v,t}, \alpha_{w,t}, \alpha_{u,t} > 0, \ \forall t$ are deterministic such that $\sum_t \alpha_{v,t} = \sum_t \alpha_{w,t} = \sum_t \alpha_{u,t} = \infty$ and $\sum_t \alpha_{v,t}^2 < \infty$, $\sum_t \alpha_{w,t}^2 < \infty$ and $\sum_t \alpha_{u,t}^2 < \infty$ with $\frac{\alpha_{u,t}}{\alpha_{v,t}} \ra 0$.
\item Define $ H(A) \doteq (A+ \transp{A})/2$ and let $\vfcn_\text{min}(C^{-1} H(A))$ be the minimum eigenvalue of the matrix $C^{-1}H(A)$. Then $\alpha_{w,t} = \eta \alpha_{v,t}$ for some $ \eta > \max{0, -\vfcn_\text{min}(C^{-1} H(A))}$.
\end{enumerate}

\begin{thmnocount}{\ref{thm:convergence}}{(Convergence of \opac)}
Let $\lambda = 0$ and consider the \opac iterations for the critic (GTD($\lambda$), i.e., TDC with importance sampling correction) and the actor (for weights $\ap_t$). Assume that (A1)-\ref{req:Alast}, (P1)-(P2) and (S1)-(S2) hold. Then the policy weights, $\ap_t$, converge to $\convhat = \{ u \in \UU \ |\   \widehat\gradtobj = 0\}$ and the value function weights, $\cp_t$, converge to the corresponding TD-solution with probability one. 
\end{thmnocount}
\begin{proof}
We follow a similar outline to that of the two timescale analysis proof for TDC (Sutton et al., 2009). We will analyze the dynamics for our two weights, $\ap_t$, and $\transp{\comb_t} = (\transp{\wvec_t} \transp{\cp_t})$, based on our update rules. We will take $\ap_t$ as the slow timescale update and $\comb_t$ as the fast inner update. 

First, we need to rewrite our updates for $\cp, \ \wvec$ and $\ap$, amenable to a two timescale analysis:
\begin{align}
&\cp_{t+1} = \cp_t + \alpha_{v,t} \rho_t [\delta_t \x_t - \gamma \transp{\x_t} \wvec \x_t]\nonumber \\
&\wvec_{t+1} = \wvec_t + \alpha_{v,t} \eta [\rho_t \delta_t \x_t - \transp{\x_t} \wvec \x_t] \nonumber\\
&\comb_{t+1} = \comb_t + \alpha_{v,t} \rho_t [G_{\ap_t,t+1} \comb_t + q_{\ap_t,t+1}] \label{eq:comb}\\
&\ap_{t+1} = \Gamma\left(\ap_t + \alpha_{\ap,t} \delta_t \frac{\grad{\pi_t(a_t|s_t)}{\ap_t}}{b(a_t|s_t)}\right) \label{eq:ap}
\end{align}
where $\rho_t=\rho(s_t,a_t)$, $\delta_t = r_{t+1} + \gamma(s_{t+1})\vfa(s_{t+1}) - \vfa(s_t)$, $\eta=\alpha_{w,t}/\alpha_{v,t}$, 
$\transp{q_{\ap_t,t+1}} = (\eta \rho_t r_{t+1} \transp{\x_t}, \ \rho_t r_{t+1} \transp{\x_t})$, and 
\[ G_{\ap_t,t+1} = \left( \begin{array}{cc}
-\eta \x_t \transp{\x_t} &  \eta \rho_t(\ap_t) \x_t \transp{(\gamma \x_{t+1} - \x_t)}\\
-\gamma \rho_t(\ap_t) \x_{t+1}\transp{ \x_t} &  \rho_t(\ap_t) \x_t \transp{(\gamma \x_{t+1} - \x_t)} \end{array} \right).\] 

Note that $G_\ap = E[G_{\ap,t}| \ap]$ and $q_\ap = E[q_{\ap,t} |\ap]$ are well defined because we assumed that the process $(\x_t, \x_{t+1}, r_{t+1})_{t \ge 0}$ is i.i.d., $0 < \rho_t \le b_{\text{min}}^{-1}$, and we have fixed $\ap_t$. Now we can define $h$ and $f$:
\begin{align*}
h(\comb_t, \ap_t) &= G_{\ap_t} \comb_t + q_{\ap_t}\\
f(\comb_t, \ap_t) &=E\left[\delta_t  \frac{\grad{\pi_t(a_t|s_t)}{\ap_t}}{b(a_t|s_t)} | \comb_t, \ap_t \right]\\
M_{t+1} &= \left(G_{\ap_t,t+1} - G_{\ap_t}\right) \comb_t + q_{\ap_t, t+1} - q_{\ap_t}\\ 
N_{t+1} &=  \delta_t  \frac{\grad{\pi_t(a_t|s_t)}{\ap_t}}{b(a_t|s_t)} - f(\comb_t, \ap_t)
\end{align*}

We have to satisfy the following conditions from Borkar (2008, p.~p64):
\begin{enumerate}[label=(B\arabic*)] 
\item $h: \Real^{N_\ap+2N_{\cp}} \ra \Real^{2N_{\cp}}$ and $f: \Real^{N_\ap+2N_\cp} \ra \Real^{N_\ap}$ are Lipschitz.\label{req:lipfcns}
\item $\alpha_{v,t}, \ \alpha_{u,t} \ \forall t$ are deterministic and $\sum_t \alpha_{v,t} = \sum_t \alpha_{u,t}= \infty, \ \sum_t\alpha_{v,t}^2 < \infty,  \ \sum_t \alpha_{u,t}^2 < \infty, \ \frac{\alpha_{u,t}}{\alpha_{v,t}} \ra 0$ (i.e.,  the system in Equation~\ref{eq:ap} moves on a slower timescale than Equation~\ref{eq:comb}).\label{req:step} 
\item The sequences $\{M_t\}_{k \ge 0}$ and $\{N_t\}_{k \ge 0}$ are Martingale difference sequences w.r.t. the increasing $\sigma$-fields, 
$\Field_t \doteq \sigma(\comb_m, \ap_m, M_m, N_m, \ m \le n)$ (i.e., $E[M_{t+1} | \Field_t] = 0$) \label{req:martingale}
\item For some constant $K > 0$, $E[ ||M_{t+1}||^2 | \Field_t] \le K(1+||x_t||^2 + ||y_t||^2)$ and $E[ ||N_{t+1}||^2 | \Field_t] \le K(1+||x_t||^2 + ||y_t||^2)$ holds for any $k \ge 0$. \label{req:martingalebound}
\item The ODE $\dot \comb(t) = h(\comb(t), \ap)$ has a globally asymptotically stable equilibrium $\vfcn(\ap)$ where $\vfcn: \Real^{N_\ap} \ra \Real^{N_\cp}$ is a Lipschitz map.\label{req:lipmap}
\item The ODE $\dot \ap(t) = f(\vfcn(\ap(t)), \ap(t))$ has  a globally asymptotically stable equilibrium, $\ap^*$.\label{req:policyparams}
\item $\sup_t (||\comb_t|| + ||\ap_t||) < \infty$, a.s. \label{req:bounded}
\end{enumerate}
An asymptotically stable equilibrium for a dynamical system is an attracting point for which small perturbations still cause convergence back to that point. 
If we can verify these conditions, then we can use Theorem 2 by Borkar (2008)  that states that $(\comb_t, \ap_t) \ra (\vfcn(\ap^*), \ap^*)$ a.s. Note that the previous actor-critic proofs transformed the update to the negative update, assuming they were minimizing costs, $-R$, rather than maximizing and so converging to a (local) minimum. This is unnecessary because we simply need to prove we have a stable equilibrium, whether a maximum or minimum; therefore, we keep the update as in the algorithm and assume a (local) maximum.

First note that because we have a bounded function, $\pi_{(:)}(s,a) : U \ra (0,1]$, we can more simply satisfy some of the properties from Borkar (2008). Mainly, we know our policy function is Lipschitz (because it is bounded and continuously differentiable), so we know the gradient is bounded, in other words,  there exists $B_{\nabla \ap} \in \Real$ such that $||\grad{\pi(a|s)}{\ap}|| \le B_{\nabla \ap}$. 

\textbf{For  requirement \ref{req:lipfcns}}, $h$ is clearly Lipschitz because it is linear in $\comb$ and $\rho_t(\ap)$ is continuously differentiable and bounded ($\rho_t(\ap) \le b_{\text{min}}^{-1}$). $f$ is Lipschitz because it is linear in $\cp$ and $\grad{\pi(a|s)}{\ap}$ is bounded and continuously differentiable (making $\obj$ with a fixed $\qfa^{\pi,\gamma}$ continuously differentiable with a bounded derivative). 


\textbf{Requirement  \ref{req:step}} is satisfied by our assumptions.

\textbf{Requirement  \ref{req:martingale}} is satisfied by the construction of $M_t$ and $N_t$. 

\textbf{For requirement  \ref{req:martingalebound}}, we can first notice that $M_t$ satisfies the requirement because $r_{t+1}, \x_t$ and $\x_{t+1}$ have uniformly bounded second moments (which is the justification used in the TDC proof (Sutton et al., 2009) and because $0 < \rho_t \le b_{\text{min}}^{-1}$. 
\begin{align*}
E&[||M_{t+1}||^2 | \Field_t]\\
&= E[|| \left(G_{\ap_t,t} - G_{\ap_t}\right) \comb_t + (q_{\ap_t,t} - q_{\ap_t}) ||^2 | \Field_t]\\ 
&\le E[|| \left(G_{\ap_t,t} - G_{\ap_t}\right) \comb_t ||^2 + ||(q_{\ap_t,t} - q_{\ap_t}) ||^2 | \Field_t]\\
&\le E[|| c_1 \comb_t ||^2 + c_2 | \Field_t]\\
&\le K(||\comb_t||^2 + 1) \le K(||\comb_t||^2 + ||\ap_t||^2 + 1)
\end{align*}
where the second inequality is by the Cauchy Schwartz inequality, $(G_{\ap_t,t} - G_{\ap_t}) \comb_t \le c_1 | \comb_t |$ and $|| q_{\ap_t,t} - q_{\ap_t} ||^2 \le c_2$  (because  $r_{t+1}, \x_t$ and $\x_{t+1}$ have uniformly bounded second moments), with $c_1, c_2 \in \Real_+$. When then simply set $K = \max(c_1, c_2)$.

For $N_t$, since the iterates are bounded as we show below for requirement \ref{req:bounded} (giving $\sup_t ||\ap_t|| < B_\ap$ and $\sup_t ||\comb_t|| < B_\comb$ for some $B_\ap, B_\comb \in \Real$. ), we see that
\begin{align*}
E&[||N_{t+1}||^2 | \Field_t]\\
&\le E\left[ \left|\left| \delta_t  \frac{\grad{\pi_t(a_t|s_t)}{\ap_t}}{b(a_t|s_t)}  \right|\right|^2 + \left|\left| E \left[\delta_t  \frac{\grad{\pi_t(a_t|s_t)}{\ap_t}}{b(a_t|s_t)} | \comb_t, \ap_t \right]  \right|\right|^2 | \Field_t\right]\\
&\le E\left[\left|\left| \delta_t  \frac{\grad{\pi_t(a_t|s_t)}{\ap_t}}{b(a_t|s_t)}  \right|\right|^2 + E \left[\left|\left| \delta_t  \frac{\grad{\pi_t(a_t|s_t)}{\ap_t}}{b(a_t|s_t)}  \right|\right|^2 | \comb_t, \ap_t\right] | \Field_t \right]\\
&\le 2 E\left[\left|\frac{\delta_t}{b(a_t|s_t)}\right|^2 ||\grad{\pi_t(a_t|s_t)}{\ap_t}||^2  | \Field_t\right]\\
&\le \frac{2}{b_{\text{min}}^2} E\left[|\delta_t|^2 B^2_{\nabla \ap}  | \Field_t\right]\\
&\le K(||\cp_t||^2 + 1) \le K(||\comb_t||^2 + ||\ap_t||^2 + 1)
\end{align*}
for some $K \in \Real$ because $E[|\delta|^2 | \Field_t] \le c_1 (1+||\cp_t||)$ for some $c_1 \in \Real$ because $r_{t+1}, \x_t$ and $\x_{t+1}$ have uniformly bounded second moments and since $||\grad{\pi(a|s)}{\ap}|| \le B_{\nabla \ap}$ $\forall \ s \in \St, \ a \in \Ac$ (as stated above because $\pi(a|s)$ is Lipschitz continuous). 

\textbf{For requirement \ref{req:lipmap}}, we know that every policy, $\pi$, has a corresponding bounded $V^{\pi,\gamma}$ (by assumption). We need to show that for each $\ap$, there is a globally asymptotically stable equilibrium of the system, $h(\comb(t), \ap)$ (which has yet to be shown for weighted importance sampling TDC, i.e., GTD($\lambda = 0$)). To do so, we use the Hartman-Grobman Theorem, that requires us to show that $G$ has all negative eigenvalues. For readability, we show this in a separate lemma (Lemma \ref{lem:gtd} below). Using Lemma \ref{lem:gtd}, we know that there exists a function $\vfcn: \Real^{N_\ap} \ra \Real^{N_\cp}$ such that $\vfcn(\ap) = \transp{(\transp{\cp_\ap} \ \transp{\wvec_\ap})}$, where $\cp_\ap$ is the unique TD-solution value-function weights for policy $\pi$ and $\wvec_\ap$ is the corresponding expectation estimate. This function, $\vfcn$, is continuously differentiable with bounded gradient (by Lemma \ref{lem:cont} below) and is therefore a Lipschitz map. 


\textbf{For requirement \ref{req:policyparams}}, we need to prove that our update $f(\vfcn(\cdot),\cdot)$ has an asymptotically stable equilibrium. This requirement can be relaxed to a local rather than global asymptotically stable equilibrium, because we simply need convergence. Our objective function, $\obj$, is not concave because our policy function, $\pi(a|s)$ may not be concave in $\ap$. Instead, we need to prove that all (local) equilibria are asymptotically stable. 

We define a vector field operator, $\hat \Gamma: \CC(\Real^{N_\ap}) \ra \CC(\Real^{N_\ap})$ that projects any gradients leading outside the compact region, $\UU$, back into $\UU$:
\begin{align*}
\hat \Gamma(g(y)) = \lim_{h \ra 0} \frac{\Gamma(y + h g(y)) - y}{h}
\end{align*}
By our forward-backward view analysis and from the same arguments following from Lemma 3 by Bhatnagar et al. (2009), we know that the ODE $\dot \ap(t) = f(\vfcn(\ap(t)), \ap(t))$ is $\gradtobj$.
Given that we have satisfied requirements 1-5 and given our step-size conditions, using standard arguments (c.f. Lemma 6 in Bhatnagar et al., 2009), we can deduce that  $\ap_t$ converges almost surely to the set of asymptotically stable fixed points, $\conva$, of $\dot \ap = \hat \Gamma \gradtobj{}$. 



\textbf{For requirement \ref{req:bounded}}, we know that $\ap_t$ is bounded because it is always projected to $\UU$. Since $\ap$ stays in $\UU$, we know that $\cp$ stays bounded (by assumption, otherwise $V^{\pi,\gamma}$ would not be bounded) and correspondingly $\wvec(\cp)$ must stay bounded, by the same argument as by Sutton et al.~(2009). Therefore, we have that $\sup_t ||\ap_t|| < B_\ap$ and that $\sup_t ||\comb_t|| < B_\rho$ for some $B_\ap, B_\comb \in \Real$. 


\end{proof}

\begin{lemma}\label{lem:gtd}
Under assumptions (A1)-\ref{req:Alast}, (P1)-(P2) and (S1)-(S2), for any fixed set of actor weights, $\ap \in \UU$, the GTD($\lambda =0$) update for the critic weights, $\cp_t$, converge to the TD solution with probability one.   
\end{lemma}
\begin{proof}
Recall that
\[ G_{\ap,t+1} = \left( \begin{array}{cc}
-\eta \x_t \transp{\x_t} &  \eta \rho_t(\ap) \x_t \transp{(\gamma \x_{t+1} - \x_t)}\\
-\gamma \rho_t(\ap) \x_{t+1} \transp{\x_t} &  \rho_t(\ap) \x_t \transp{(\gamma \x_{t+1} - \x_t)} \end{array} \right).\] 
and $G_\ap = \E{}{G_{\ap,t}}$, meaning
\[ G_\ap = \left( \begin{array}{cc}
-\eta C &  -\eta \Arho\\
-\transp{F_\rho(\ap)} & -\Arho \end{array} \right).\] 

where $F_\rho(\ap) = \gamma \E{}{\rho_t(\ap) \x_{t+1} \transp{\x_t}}$, with $C_\rho(\ap) = \Arho - F_\rho(\ap)$. For the remainder of the proof, we will simply write $A_\rho$ and $C_\rho$, because it is clear that we have a fixed $\ap \in \UU$. 

Because $GTD(\lambda)$ is solely for value function approximation, the feature vector, $\x$, is only dependent on the state:
\begin{align*}
\E{}{\rho_t \x_t \transp{\x_t}} &=\sum_{s_t,a_t} d(s_t) b(a_t|s_t) \rho_t \x(s_t)\transp{\x_t} \\
&=\sum_{s_t,a_t} d(s_t) \pi(a_t|s_t) \x(s_t) \transp{\x_t} \\
&=\sum_{s_t} d(s_t) \x(s_t) \transp{\x_t}  \left(\sum_{a_t} \pi(a_t|s_t)\right) \\
&=\sum_{s_t}  d(s_t) \x(s_t) \transp{\x_t} =E[\x_t \transp{\x_t} ]
\end{align*}
because $\sum_{a_t} \pi(a_t|s_t)=1$. A similar argument shows that $\E{}{\rho_t \x_{t+1} \transp{\x_t}} =\E{}{\x_{t+1} \transp{\x_t}}$. Therefore, we get that $F_\rho(\ap) = \gamma E[\x \transp{\x_t} ]$ and $\Arho = E[\x_t \transp{(\gamma \x_{t+1} - \x_t)}]$. The expected value of the update, $G$, therefore, is that same as for TDC, which has been shown to converge under our assumptions (see Maei, 2011). 

\end{proof}

\begin{lemma}\label{lem:cont}
Under assumptions (A1)-\ref{req:Alast}, (P1)-(P2) and (S1)-(S2), let $\vfcn: \UU \ra \mathcal{V}$ be the map from policy weights to corresponding value function, $V^{\pi,\gamma}$, obtained from using GTD($\lambda =0$) (proven to exist by Lemma \ref{lem:gtd}). Then $\vfcn$ is continuously differentiable with a bounded gradient for all $\ap \in \UU$. 
\end{lemma}
\begin{proof}
To show that $\vfcn$ is continuous, we use the Weierstrass definition ($\delta - \epsilon$ definition). Because $\vfcn(\ap) = -G(\ap)^{-1} q(\ap) = \comb_\ap$, which is a complicated function of $\ap$, we can luckily break it up and prove continuity about parts of it. Recall that 1) the inverse of a continuous function is continuous at every point that represents a non-singular matrix and 2) the multiplication of two continuous functions is continuous. Since $G(\ap)$ is always nonsingular, we simply need to proof that $a(\ap) \ra G(\ap)$ and $b(\ap) \ra q(\ap)$ are continuous. $G(\ap)$ is composed of several block matrices, including $C$, $F_\rho(\ap)$ and $\Arho$. We will start by showing that $\ap \ra F_\rho(\ap)$ is continuous, where $F_\rho(\ap) = -\E{b}{\eta \rho_t(\ap)\x_{t+1} \transp{\x_t}}$. The remaining entries are similar. 

Take any $s \in \St$, $a \in \Ac$, and $\ap \in \UU$. We know that $\pi(a|s): \UU \ra [0,1]$ is continuous for all $\ap \in \UU$ (by assumption). Let $\epsilon_1 = \frac{\epsilon}{\gamma |\Ac| \E{b}{\x_{t+1} \transp{\x_t}}}$ (well-defined because $\E{b}{\x_{t+1} \transp{\x_t}}$ is nonsingular). Then we know there exists a $\delta > 0$ such that for any $\ap_2 \in \UU$ with $||\ap_1 - \ap_2|| < \delta$, then $||\pi_{\ap_1}(a_t | s_t) - \pi_{\ap_2}(a_t | s_t)|| < \epsilon_1$. Now
\begin{align*}
||F_\rho(\ap_1) - F_\rho(\ap_2)|| &= \gamma || \E{}{\rho_t(\ap_1) \x_{t+1} \transp{\x_t}} - \E{}{\rho_t(\ap_2) \x_{t+1} \transp{\x_t}}||\\
&= \gamma \left|\left| \sum_{s_t, a_t} d^b(s_t) b(a_t | s_t) \frac{\pi_{\ap_1}(a_t | s_t)}{b(a_t | s_t)} \x_{t+1} \transp{\x_t} - \sum_{s_t, a_t} d^b(s_t) b(a_t | s_t) \frac{\pi_{\ap_2}(a_t | s_t)}{b(a_t | s_t)}  \x_{t+1} \transp{\x_t} \right|\right|\\
&= \gamma \left|\left| \sum_{s_t, a_t} d^b(s_t) [\pi_{\ap_1}(a_t | s_t) - \pi_{\ap_2}(a_t | s_t)] \x_{t+1} \transp{\x_t} \right|\right|\\
&< \gamma \sum_{s_t, a_t} d^b(s_t) ||\pi_{\ap_1}(a_t | s_t) - \pi_{\ap_2}(a_t | s_t)|| \x_{t+1} \transp{\x_t}\\
&< \gamma \epsilon_1 \sum_{s_t, a_t} d^b(s_t) \x_{t+1} \transp{\x_t}\\
&= \gamma \epsilon_1 |\Ac| \E{b}{\x_{t+1} \transp{\x_t}} = \epsilon
\end{align*}
Therefore, $\ap \ra F_\rho(\ap)$ is continuous. This same process can be done for $A_\rho(\ap)$ and $\E{b}{\rho_t(\ap) r_t \x_t}$ in $q(\ap)$.

Since $\ap \ra G$ and $\ap \ra q$ are continuous for all $\ap$, we know that $\vfcn(\ap) = -G(\ap)^{-1} q(\ap)$ is continuous. 

The above can also be accomplished to show that $\grad{\vfcn}{\ap}$ is continuous, simply by replacing $\pi$ with $\grad{\pi}{\ap}$ above. Finally, because our policy function is Lipschitz (because it is bounded and continuously differentiable), we know that it has a bounded gradient. As a result, the gradient of $\vfcn$ is bounded (since we have nonsingular and bounded expectation matrices), which would again follow from a similar analysis as above. 
\end{proof}

\clearpage
\section{Errata}
\label{errata}

The current theoretical results only apply to tabular representations for the
policy $\pi$ and not necessarily to function approximation for the policy. Thanks to Hamid Reza Maei for pointing out this issue. 
We are working on correcting this issue, both by re-examining the
current theoretical results and working on modifications to the algorithm. Note, however,
that even theoretical results restricted to tabular representations indicate that the algorithm has a principled
design. 

The are two mistakes in the theoretical analysis that cause us to scale back the claims.
The first problem is about the existence of stable minima for our approximate gradient.
 Because the approximate gradient
 is not the gradient of any objective function, it is not clear if any stable minima 
 exist. To justify the existence of these minima, we need to prove that when $\gradtobj = 0$, perturbations
 to $\ap$ push $\ap$ back to this minimum (rather than away from it). To do so, we will need
 to consider linearizations around the minimum using the Hartman-Grobman theorem.
 For a tabular representation, this is not a problem because each iteration strictly improves
 the value via a local change to $\pi(a | s)$ for a specific state and action,
 without aliasing causing ripple effects (as we showed in Theorem 1). 

The second error is in the proof of the claim in Theorem 2 that $\conv \subseteq \conva$. It remains true that $\conv = \conva$ for tabular representations.
In the Policy Improvement claim in Theorem 1, we can say that the policy update overall improves the policy, allotting more importance
to states more highly weighted by the \textit{stationary behavior distribution, $d^b$}. Therefore, it is still true that
\begin{align*}
\obj(\ap_t) 
&\leq \sum_{s \in \St} d^b(s) \sum_{a \in \Ac} \pi_{\ap'}(a|s) \E{\pi_{\ap'},\gamma}{r_{t+1} + \gamma_{t+1} V^{\pi_{\ap},\gamma}(s_{t+1})}
\end{align*}
Expanding this further, we get
\begin{align*}
\obj(\ap_t) 
&\leq \sum_{s \in \St} d^b(s) \sum_{a \in \Ac} \pi_{\ap'}(a|s) \sum_{s,a,s_{t+1}} P(s,a,s_{t+1})  \left[ R(s,a,s_{t+1}) + \gamma_{t+1} V^{\pi_{\ap},\gamma}(s_{t+1}) \right]\\
&= \sum_{s \in \St} d^b(s) \sum_{a \in \Ac} \pi_{\ap'}(a|s) \sum_{s,a,s_{t+1}} P(s,a,s_{t+1}) \\
& \ \ \ \ \ \ \left[ R(s,a,s_{t+1}) +  \gamma_{t+1} \sum_{a_{t+1}} \pi_\ap(a_{t+1} | s_{t+1}) \sum_{s_{t+2}} P(s_{t+1},a_{t+1},s_{t+2}) \left[R(s_{t+1},a_{t+1},s_{t+2} + \gamma_{t+2} V^{\pi_{\ap},\gamma}(s_{t+2}) \right] \right]
\end{align*}
We know that
\begin{align*}
\sum_{s_{t+1}} &d^b(s_{t+1}) \sum_{a_{t+1}} \pi_\ap(a_{t+1} | s_{t+1}) \sum_{s_{t+2}} P(s_{t+1},a_{t+1},s_{t+2}) \left[R(s_{t+1},a_{t+1},s_{t+2} + \gamma_{t+2} V^{\pi_{\ap},\gamma}(s_{t+2}) \right] \\
&\le \sum_{s_{t+1}} d^b(s_{t+1}) \sum_{a_{t+1}} \pi_{\ap'}(a_{t+1} | s_{t+1}) \sum_{s_{t+2}} P(s_{t+1},a_{t+1},s_{t+2}) \left[R(s_{t+1},a_{t+1},s_{t+2} + \gamma_{t+2} V^{\pi_{\ap},\gamma}(s_{t+2}) \right]
\end{align*}

We see that unless $P(s,a, \cdot)$ and $d^b$ are similar, then the next important inequality, in the middle of the proof of Theorem 1, with actions selected according to $\pi_{\ap'}$, \textbf{might not hold:}
\begin{align*}
&\sum_{a} \pi_{\ap'}(a | s) \sum_{s,a,s_{t+1}} P(s,a,s_{t+1}) \sum_{a_{t+1}} \pi_\ap(a_{t+1} | s_{t+1}) \sum_{s_{t+2}} P(s_{t+1},a_{t+1},s_{t+2}) \left[R(s_{t+1},a_{t+1},s_{t+2} + \gamma_{t+2} V^{\pi_{\ap},\gamma}(s_{t+2}) \right] \\
&\le \sum_a \pi_{\ap'}(a | s) \sum_{s,a,s_{t+1}} P(s,a,s_{t+1}) \sum_{a_{t+1}} \pi_{\ap'}(a_{t+1} | s_{t+1}) \sum_{s_{t+2}} P(s_{t+1},a_{t+1},s_{t+2}) \left[R(s_{t+1},a_{t+1},s_{t+2} + \gamma_{t+2} V^{\pi_{\ap},\gamma}(s_{t+2}) \right]
\end{align*}
For tabular representations,
this weighting is not relevant because the policy is updated individually for each state. As the representation becomes less local, the
weighting by $d^b$ becomes more relevant for trading off error in different states. Moreover, for $b = \pi$, i.e. on-policy, the policy improvement claim is true. As the stationary distributions induced by $b$ and $\pi$ become more different, and there is significant aliasing in the policy
function approximator, then it is unclear if updates that improve $\pi$ according to weighting states by $d^b$ improve $\pi$ according to
weighting states by $d^\pi$.

\end{document}